\documentclass[runningheads]{llncs}
\usepackage[T1]{fontenc}
\usepackage{graphicx}
\usepackage{booktabs}
\usepackage{threeparttable}
\usepackage[misc]{ifsym}
\newcommand{\corr}{(\Letter)}
\usepackage{mwe}

\usepackage[T1]{fontenc}
\usepackage{graphicx}
\usepackage{booktabs}

\usepackage{amsmath}
\usepackage{amsfonts}
\usepackage{amssymb}
\usepackage{bm}
\usepackage{dsfont}
\usepackage{mathtools}
\usepackage{algorithm}
\usepackage{algpseudocode}

\usepackage{array}
\usepackage{multirow}
\usepackage{arydshln}
\usepackage[export]{adjustbox}
\DeclareGraphicsExtensions{.pdf,.png,.jpg,.jpeg}

\usepackage{xcolor}
\usepackage{tikz}
\usepackage{pgfplots}
\pgfplotsset{compat=1.18} 

\usepackage{textcomp}
\usepackage{url}
\usepackage{verbatim}
\usepackage{float}
\usepackage{comment}
\usepackage{soul}
\usepackage{cite}
\usepackage{hyperref}

\usepackage{physics} 
\usepackage{graphicx}

\makeatletter
\@ifundefined{prop}{\spnewtheorem{prop}{Proposition}{\bfseries}{\itshape}}{}
\@ifundefined{assum}{\spnewtheorem{assum}{Assumption}{\bfseries}{\itshape}}{}
\makeatother

\newcommand{\bmphi}{\bm{\phi}}

\newcommand{\vecz}{\mathbf{z}}

\newcommand{\vecY}{\mathbf{y}}
\newcommand{\matZ}{\mathbf{Z}}

\begin{document}

\title{Informative Perturbation Selection for Uncertainty-Aware Post-hoc Explanations}
\titlerunning{EAGLE: Active Sampling for Explanations}

\toctitle{Informative Perturbation Selection for Uncertainty-Aware Post-hoc Explanations}
\tocauthor{Sumedha Chugh, Ranjitha Prasad, Nazreen Shah}
\author{Sumedha Chugh\inst{1}\corr \and
Ranjitha Prasad\inst{1} \and
Nazreen Shah\inst{1}}


\authorrunning{Sumedha Chugh, Ranjitha Prasad, Nazreen Shah}


\institute{
Indraprastha Institute of Information Technology Delhi, India \\
\email{sumedhac@iiitd.ac.in}
}

\maketitle              

\begin{abstract}
Trust and ethical concerns due to the widespread deployment of opaque machine learning (ML) models motivate the need for reliable model explanations. Post-hoc model-agnostic explanation methods address this challenge by learning a surrogate model that approximates the behavior of the deployed black-box ML model in the locality of a sample of interest. In post-hoc scenarios, neither the underlying model parameters nor the training data are available, and hence, this local neighborhood must be constructed by generating perturbed inputs in the neighborhood of the sample of interest, and its corresponding model predictions. We propose \emph{Expected Active Gain for Local Explanations} (\texttt{EAGLE}), a post-hoc model-agnostic explanation framework that formulates perturbation selection as an information-theoretic active learning problem. By adaptively sampling perturbations that maximize the expected information gain, \texttt{EAGLE} efficiently learns a linear surrogate explainable model while producing feature importance scores along with the uncertainty/confidence estimates. Theoretically, we establish that cumulative information gain scales as $\mathcal{O}(d \log t)$, where $d$ is the feature dimension and $t$ represents the number of samples, and that the sample complexity grows linearly with $d$ and logarithmically with the confidence parameter $1/\delta$. Empirical results on tabular and image datasets corroborate our theoretical findings and demonstrate that \texttt{EAGLE} improves explanation reproducibility across runs, achieves higher neighborhood stability, and improves perturbation sample quality as compared to state-of-the-art baselines such as Tilia, US-LIME, GLIME and BayesLIME.

\end{abstract}
\section{Introduction}
\label{sec:intro}
\noindent Explainable AI (XAI) seeks to make predictions interpretable, a goal reinforced by regulations such as the EU AI Act and GDPR \cite{GDPR}. Due to the growing complexity of modern systems and the use of pre-trained models, post-hoc approaches are often preferred for their broader applicability and scalability \cite{posthoc}. Popular post-hoc model-agnostic approaches such as LIME \cite{lime}, US-LIME \cite{uslime}, BayLIME \cite{zhao2021baylime} and GLIME \cite{GLIME2023} explain complex model predictions by fitting local surrogate models using constraints or prior assumptions in sampling strategies. Explanations provided using surrogate-based methods are statistical estimates, and hence, uncertainty naturally arises due to (a) finite perturbation samples, (b) randomness in sampling methods used, (c) surrogate model approximation error and (d) noise that arises in model predictions. Consequently, explanations may vary significantly across instantiations and fail to capture the local behavior of the black-box models. Hence, explanations without measures of confidence alongside may be misleading. This motivates  the need for principled Bayesian approaches that explicitly quantify uncertainty in local explanations \cite{zhao2021baylime,slack2021reliable}.

\noindent Among Bayesian approaches, UnRAvEL \cite{unravel} utilizes Bayesian optimization for  perturbation sampling via acquisition functions that balance exploration and exploitation, but this strategy can introduce sampling bias due to repeated selection of points close to the local instance. Several papers in the literature employ Bayesian linear regression as the surrogate model because this framework naturally captures uncertainty in feature importance through posterior distributions while being computationally tractable and interpretable. BayesLIME \cite{slack2021reliable} uses the posterior distributions over feature importance and selects perturbations based on predictive variance as a proxy for uncertainty in local predictions, making it one of the few methods that  leverages uncertainty for perturbation selection. However, the sampling strategy in BayesLIME is heuristic and does not explicitly incorporate locality information when evaluating candidate perturbations. This raises a natural  question - \emph{Can we derive principled sample-efficient perturbation selection strategies that directly minimize explanation uncertainty while preserving locality, leading to improved local fidelity?}

\noindent A principled approach to improving local explanation stability requires distinguishing between two fundamentally different sources of uncertainty, aleatoric and epistemic \cite{NIPS2017_7141}. For perturbation-based post-hoc explanation methods, sampling driven by aleatoric noise is ineffective as it only captures irreducible variability in data instead of reducing the epistemic uncertainty of the surrogate model used to obtain feature attributions. This motivates the use of active learning which focuses on selecting maximally informative samples in regions of high epistemic uncertainty \cite{rethinkingunc,activele}. 

\noindent \textbf{Contributions}: We propose the \emph{Expected Active Gain for Local Explanations} (\texttt{EAGLE}) framework, which leverages an expected information gain active learning criterion to guide perturbation selection toward maximally informative regions. In particular, for the linear Bayesian surrogate formulation, we propose a novel acquisition function, which selects perturbations that maximize the expected reduction in posterior uncertainty over the explanation coefficients while retaining locality information during perturbation selection, ensuring that sampling remains focused on the neighborhood of the instance of interest. Based on the posterior covariance matrix of the Gaussian surrogate model, we provide a theoretical analysis of the proposed acquisition strategy. In particular, (a) we characterize the growth rate of the cumulative information gain and (b) derive high-probability bounds on the estimation error of the explanation weights. This analysis leads to sample complexity guarantees for the proposed approach. We provide empirical results that validate the theoretical findings and demonstrate the practical effectiveness of \texttt{EAGLE}. Experiments on image datasets such as MNIST, ImageNet, together with tabular datasets including COMPAS, German Credit, Adult Income, and Magic (Gamma Telescope), demonstrate that \texttt{EAGLE} improves sampling quality and enhances explanation stability as compared to the state of the art baselines, based on metrics such as Jaccard similarity, D-efficiency and Cumulative information gain. These results empirically corroborate our theoretical analysis, which establishes information gain bounds and sample complexity guarantees for the proposed acquisition strategy.

\begin{figure}[t]
    \centering
    
        \includegraphics[width=.9\textwidth]{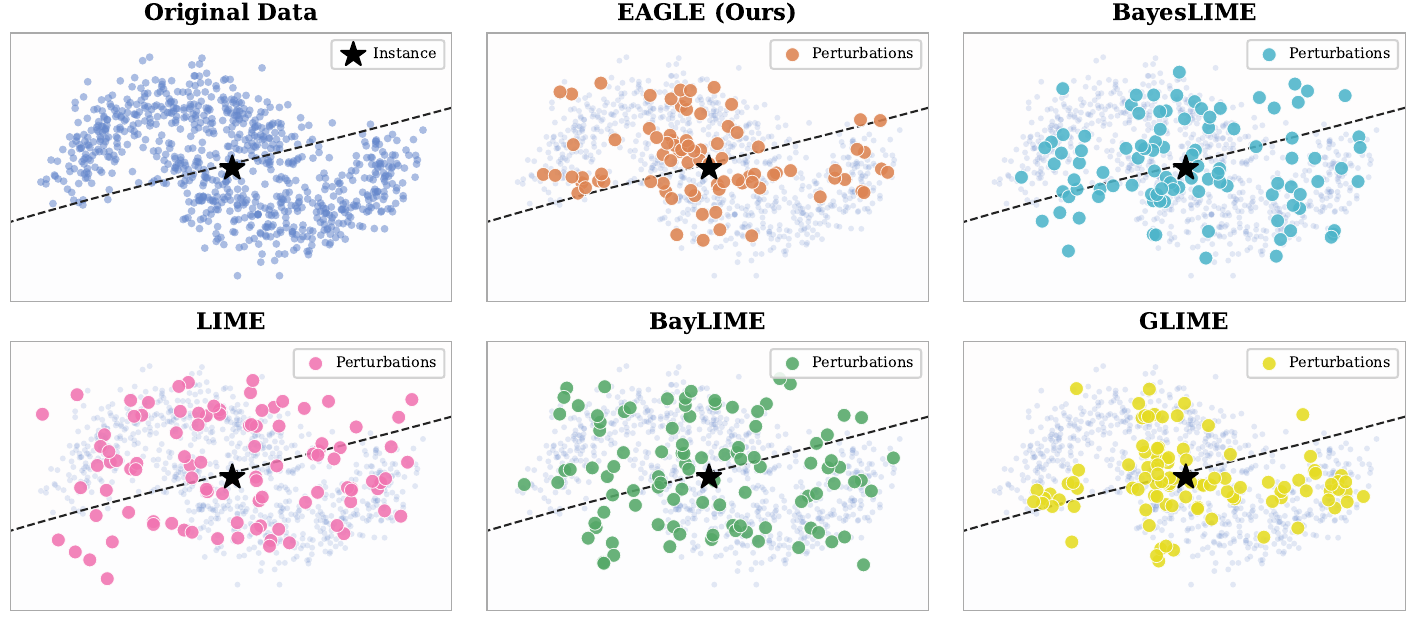}
    \caption{Perturbation strategies on the make\_moons \cite{pedregosa2011scikit} dataset ($n = 100$). Existing methods either sample in the vicinity of the instance disregarding the regions of epistemic uncertainty (LIME, BayLIME), are constrained excessively close the  instance of interest (GLIME), or fail to adapt to the locality and only focus on predictive variance (Focus Sampling/BayesLIME). In contrast, \texttt{EAGLE} (ours) selects perturbations that respect both locality and regions of high epistemic uncertainty while covering both sides of the decision boundary, resulting in a compact and  informative neighborhood.}
    \label{fig:motivation}
    
\end{figure}
\section{Related Works}
\label{sec:relatedworks}
\noindent \textbf{Perturbations based local post-hoc explainers:}~Perturbations based local post-hoc explainers obtain instance level interpretations by generating samples in the neighborhood of a target input and fitting a simple, interpretable model on the resulting surrogate dataset. Popular linear surrogate methods include LIME~\cite{lime} and DLIME \cite{dlime}, which learns a sparse linear surrogate under a proximity based kernel weighting. While these methods promise simplicity and flexibility, they rely on random perturbation generation based on several design choices on kernel width, sampling distributions, etc. This results in instability as explanations can vary across runs, even for the same instance and model~\cite{alvarez2018robustness}. GLIME partially solved the issue of instability by constraining sampling to a specific predefined region \cite{GLIME2023}. Subsequent work such as  Tilia~\cite{tilia} moved away from  linear surrogate models to decision trees in order to provide structured explanations. CLIMAX \cite{nanavati2023climax} takes a different direction by replacing the local regression surrogate with a local classifier-based explainer to generate contrastive explanations and using label aware surrogate data generation to improve stability and sample efficiency. Notably, these methods output a single attribution vector as a point estimate; they do not explicitly quantify uncertainty or provide any notions of explanation reliability. 

\noindent \textbf{Stable Local Explanations via Guided Sampling:}~Several works in literature have focused on improving the stability and fidelity of local explanations by  controlling the random perturbation strategy. For instance, GLIME~\cite{GLIME2023} extends LIME by sampling directly from a locality focused distribution to achieve faster convergence and consistent explanations.  US-LIME~\cite{uslime} improves LIME by selecting perturbations that are both close to the decision boundary (via uncertainty sampling) and close to the target instance. BayLIME \cite{zhao2021baylime} uses priors over feature attributions for stability, without considering them for perturbation selection. UnRAvEL~\cite{unravel} uses a Gaussian process (GP) within a Bayesian optimization (BO) surrogate and a novel  acquisition to trade off local fidelity and information gain. While these methods collectively demonstrate that informative or constrained sampling  improves stability of explanations, they  rely on heuristics or fixed sampling criteria and hence, they do not eliminate instability.

\noindent \textbf{Uncertainty aware explanations:}~ Knowing how much to trust an explanation is as important as the explanation itself. Broadly,  includes methods that study explanation variability under Bayesian neural networks, attribute conformal interval uncertainty to input features, or analyse explanation ambiguity arising from model multiplicity \cite{antoran2021clue,bykov2025explaining,ilidrissi2025unveil,marx2023sure}. These directions are important but are not directly comparable to our setting.
Particularly relevant to our setting are local post hoc methods that augment feature importance values with uncertainty or confidence estimates. Among Bayesian approaches, UnRAvEL \cite{unravel} utilizes Bayesian optimization for  perturbation sampling. A well known uncertainty-aware post hoc framework is\cite{slack2021reliable}, which instantiates Bayesian variants of LIME and KernelSHAP, namely BayesLIME and BayesSHAP, to provide credible intervals for feature contributions; in addition, it proposes focused sampling by selecting perturbations with high posterior predictive variance. Beyond Bayesian attribution methods, a complementary calibration based stream is represented by Calibrated Explanations \cite{lofstrom2024calibrated} and Fast Calibrated Explanations \cite{lofstrom2026fast}, which incorporate uncertainty into local feature importance explanations by providing calibrated probability estimates and uncertainty information for feature weights.

\noindent \textbf{Novelty:}~The proposed \texttt{EAGLE} framework leverages an information-theoretic acquisition strategy to maximize the expected information gain in the samples, leading to the reduction in uncertainty in the surrogate explanation model. This strategy lets us prioritize sample queries  that naturally target regions where the surrogate epistemic uncertainty is high, enabling the method to explore informative areas of the perturbation space. To the best of the authors’ knowledge, \texttt{EAGLE} is the first framework to cast perturbation selection for uncertainty aware local explanations as a principled information theoretic active learning problem.
\section{Problem Setting and Bayesian Surrogate Modeling}
\label{sec: problem_setting}
Let $f_b: \mathbb{R}^d \to [0,1]$ denote a black-box classifier that outputs a label $y_i$ for input 
$\mathbf{x}_i \in \mathbb{R}^d$. The goal of a post-hoc model agnostic explainer is to explain the prediction $y_0 = f_b(\mathbf{x}_0)$ for a particular instance $\mathbf{x}_0$. A set of $N$ perturbations around $\mathbf{x}_0$, denoted by $\mathbf{Z} = \{\mathbf{z}_i\}_{i=1}^N$ where $\mathbf{z}_i \in \mathbb{R}^d$ with their black-box predictions $f_b(\mathbf{z}_i)$ is employed to develop an interpretable surrogate model $f_e$. 
To enforce locality, $\mathbf{z}_i$ is assigned a proximity weight $\pi_{\mathbf{x}_0}(\mathbf{z}_i)$ based on its distance from $\mathbf{x}_0$. Subsequently, $f_e$ is trained by minimizing the locality-aware weighted loss
\begin{align}
\label{eq:eagle_weighted_loss}
\mathcal{L}(f_e,f_b,\pi_{\mathbf{x}_0}) 
= \sum_{i=1}^N \pi_{\mathbf{x}_0}(\mathbf{z}_i)\,
\big(f_b(\mathbf{z}_i) - f_e(\mathbf{z}_i)\big)^2. 
\end{align} 
Finally, the explanation is represented by feature importance scores $\boldsymbol{\phi} \in \mathbb{R}^d$ derived from $f_e$. A popular example is LIME, where $f_e$ is linear and $\boldsymbol{\phi}$ corresponds to the coefficients of the linear model \cite{lime}.

\noindent \textbf{Bayesian Surrogate Formulation:}~ Uncertainty-aware approaches \cite{slack2021reliable} model the surrogate $f_e$ using Bayesian linear regression as:
\begin{align}
\label{eq:bayesian_surrogate}
    f_e(\mathbf{z}_i) &= \mathbf{z}_i^\top \bmphi + \epsilon_i, 
    \qquad \epsilon_i \sim \mathcal{N}\!\left(0, \tfrac{\sigma^2}{\pi_{\mathbf{x}_0}(\mathbf{z}_i)}\right),
\end{align}
where $\bmphi \in \mathbb{R}^d$ are regression coefficients representing feature contributions. Following prior work, we define a locality weighting function $\pi_{\mathbf x_0}(\mathbf z)$ controlling the size of the local neighborhood around the instance of interest $\mathbf x_0$. Specifically, the locality weight $\pi_{\mathbf{x}_0}(\mathbf{z}_i)$ ensures that perturbations closer to $\mathbf{x}_0$ are modeled with higher precision, while allowing larger variance for distant points. Placing conjugate priors on $\bmphi$ and $\sigma^2$ where $\bmphi \mid \sigma^2 \sim \mathcal{N}(0, \sigma^2 \mathbf{I}_d)$, and $\sigma^2 \sim \operatorname{Inv}\text{-}\chi^2(n_0, \sigma_0^2)$, we obtain posterior distributions as 
\begin{align*}
   &\bmphi \mid \sigma^2, \mathcal{Z}, \vecY \sim \mathcal{N}(\hat{\bmphi}, \mathbf{V}_{\bmphi}\sigma^2), \quad \sigma^2 \mid \mathcal{Z}, \vecY \sim \operatorname{Scaled\text{-}Inv}\text{-}\chi^2 \!\left(n_0+N,\;\frac{n_0\sigma_0^2 + s^2}{n_0+N}\right),
\end{align*}
 where \(\vecY = [y_1,\dots,y_N]^\top\) denotes the vector of black-box responses for the perturbations $\{\mathbf z_i\}_{i=1}^N \in \mathcal{Z}$, with \(y_i = f_b(\mathbf z_i)\). Equivalently, the likelihood can be written in terms of the weighted design matrix
\( \matZ \in \mathbb{R}^{N \times d} \), whose rows correspond to perturbations
\( \mathbf z_i \in \mathcal{Z} \), together with the diagonal weight matrix
\( \mathbf W = \mathrm{diag}\!\left(\operatorname{\Pi}_{\mathbf x_0}(\matZ)\right) \).
Here, \( \operatorname{\Pi}_{\mathbf x_0}(\matZ) \) denotes the vector of locality kernel evaluations
applied row wise to \( \matZ \). Further, the posterior mean is given by $\hat{\bmphi} = \mathbf{V}_{\bmphi}\,\matZ^\top \mathbf{W} \vecY$, where
 \begin{align}
     \mathbf{V}_{\bmphi} &= \big(\matZ^\top \mathbf{W} \matZ + \mathbf{I}_d \big)^{-1} \quad s^2 = (\vecY-\matZ\hat{\bmphi})^\top \mathbf{W} (\vecY-\matZ\hat{\bmphi}) + \hat{\bmphi}^\top \hat{\bmphi}.
     \label{eq:covarianceBLR}
 \end{align}
 
\noindent The posterior mean $\hat{\bmphi}$ serves as the local feature importance scores.  
The posterior predictive distribution for a new perturbation $\mathbf{z}_i$ is a Student-$t$ given as 
\begin{align}
  \hat{y}(\mathbf{z}_i) \mid \matZ, \vecY \sim t_{\nu = n_0+N}\!\left(\hat{\bmphi}^\top \mathbf{z}_i,\; \big(1+\mathbf{z}_i^\top \mathbf{V}_{\bmphi}\mathbf{z}_i\big)s^2\right),  
\end{align}
whose variance  is given as $\text{var}(\hat{y}(\vecz_i)) = ((\vecz_i^T \mathbf{V}_{\bmphi}\vecz_i+1)s^2)\frac{\nu}{\nu-2}$, where $n_0$ is a prior parameter. In \cite{slack2021reliable}, the authors use the predictive variance, $\text{var}(\hat{y}(\vecz_i))$ as a proxy for local model uncertainty.


\section{Proposed Approach: \texttt{EAGLE}}
\label{sec:proposedTech}

We introduce the \emph{Expected Active Gain for Local Explanations} (\texttt{EAGLE}) framework which adopts an active learning strategy to iteratively select the next sample perturbation \( \mathbf{z}^\star \) from a candidate pool \( \mathcal{P} \) by maximizing an
acquisition function \( \mathcal{A}(\mathbf{z}; f_e) \), which quantifies the
informativeness of a candidate perturbation \( \mathbf{z} \) for improving the surrogate
model \( f_e \):
\begin{align}
\mathbf{z}^\star
=
\arg\max_{\mathbf{z} \in \mathcal{P}}
\mathcal{A}_{\text{E}}(\mathbf{z}; f_e).
\end{align}

We instantiate the \texttt{EAGLE} acquisition function using an information-theoretic
criterion that maximizes the expected
information gain in the parameters as:
\begin{align}
\mathcal A_{\text{E}}({\mathbf z};f_e)
=
\mathbb{E}_{y \mid \mathcal Z,{\mathbf z}}
\left[
H\big(\bmphi \mid \mathcal Z\big)
-
H\big(\bmphi \mid \mathcal Z \cup \{({\mathbf z},y)\}\big)
\right],
\label{eq:EAGLEentropydef}
\end{align}
i.e., $\mathcal A_{\text{E}}({\mathbf z};f_e)$ selects the sample ${\mathbf z}$ that result in the maximum expected reduction in entropy of the surrogate
parameters $\bmphi$. In the supplementary, we derive the acquisition function \( \mathcal{A}_{\text{E}}(\mathbf Z;f_e) \), which characterizes the expected information gain of a set of perturbations in $\mathbf Z$. Using such an acquisition function is complex, it is not directly suitable for sequential (greedy) or batch sample selection. In the following theorem, we provide the acquisition function under single-step greedy approach to
maximize the predictive variance. 

\begin{theorem}
\label{th: acquisition}
Consider a Bayesian linear surrogate model given in \eqref{eq:bayesian_surrogate}
and the posterior covariance in \eqref{eq:covarianceBLR}.
Let the \texttt{EAGLE}-based acquisition function be defined as in \eqref{eq:EAGLEentropydef}. Then, under single-step greedy acquisition, maximizing
\( \mathcal{A}_{\mathrm{E}}(\mathbf z) \) over a candidate pool is equivalent,
up to additive and multiplicative constants independent of \( \mathbf z \), to the following:
\begin{align}
\arg\max_{\mathbf z} \mathcal{A}_{\mathrm{E}}(\mathbf z; f_e)
=
\arg\max_{\mathbf z} \; \pi_{\mathbf x_0}(\mathbf z)\ \mathbf z^\top \mathbf V_{\bm\phi} \mathbf z.
\label{eq:finalAcq}
\end{align}
\end{theorem}
Theorem~\ref{th: acquisition} shows that, for a Bayesian linear surrogate model, maximizing the expected information gain used in the \texttt{EAGLE} framework is equivalent to selecting the perturbations that maximize the locality-weighted posterior uncertainty.

\section{Theoretical Analysis: \texttt{EAGLE}}

In this section, we analyze the sample complexity of weighted expected information gain driven perturbation selection and characterize the rate at which uncertainty in the surrogate explanation decreases. We consider a Bayesian linear regression surrogate defined in \eqref{eq:bayesian_surrogate} over a set of
perturbations \( \mathbf Z = [\mathbf z_1,\dots,\mathbf z_N] \), where the
surrogate prediction at a perturbation \( \mathbf z \) is given by $\hat y(\mathbf z) = \mathbf z^\top \bmphi$. Further, let \( \bmphi^\star \in \mathbb R^d \) denote the true (unknown)
importance vector representing the ground-truth local explanation.
We assume predictions pertaining to the $t$-th query sample are generated according to the linear model
\begin{align}
y_t = \mathbf z_t^\top \bmphi^\star + \varepsilon_t,
\qquad
\varepsilon_t \sim \mathcal N(0,\sigma^2),
\label{eq:sourceLinModel}
\end{align}
and that perturbations satisfy the boundedness condition
\(
\|\mathbf z_t\|_2 \le 1
\)
for all \(t\). From the previous section, we know that under this model, the posterior precision matrix for the perturbations in $\mathbf Z$ is given by $\mathbf V_{\bmphi}$. Here, we introduce a notation to represent the covariance matrix for after $t$ perturbations as $\mathbf V_{\bmphi,t}$. From \eqref{eq:covarianceBLR}, the covariance matrix updated for the $t$-th sample is given as $\mathbf V^{-1}_{\bmphi,t} = \mathbf V_{\bmphi, t-1}^{-1}
+
\pi_{\mathbf x_0}(\mathbf z)\,
\mathbf z_t \mathbf z_t^\top$, where $\mathbf z_t$ refers to the sample obtained using $\mathcal A_{\text{E}}$ as in \eqref{eq:finalAcq}. Expanding $\mathbf V_{\bmphi, t-1}^{-1}$ recursively and using $\mathbf V_{\bmphi, 0} = \mathbf{I}$, we obtain
\begin{align}
\mathbf V_{\bmphi, t}^{-1}
=
\mathbf I
+
\sum_{s=1}^t
\pi_{\mathbf x_0}(\mathbf z_s)\,
\mathbf z_s \mathbf z_s^\top.
\label{eq:posteriortqueries}
\end{align}
This formulation allows us to analyze how weighted EIG-based sampling
reduces uncertainty in the explanation parameters and induces
concentration of the posterior mean around the ground-truth
importance vector \( \bmphi^\star \).

\begin{algorithm}[t]
\caption{\texttt{EAGLE}: Expected Active Gain for Local Explanations}
\begin{algorithmic}[1]
\Require Black-box $f_b$, instance $\mathbf{x}_0$, seed $S$, batch size $B$, pool size $A$, budget $N$
\State Generate $S$ seed perturbations $\mathcal{Z} = \{\mathbf{z}_1, \dots, \mathbf{z}_S\}$ around $\mathbf{x}_0$
\State Query $\mathbf{y} \gets [f_b(\mathbf{z}_1), \dots, f_b(\mathbf{z}_S)]^\top$
\State Fit surrogate: $\hat{\bmphi},\mathbf{V}_{\bmphi} \gets \text{BLR}(\mathcal{Z}, \mathbf{y}, \pi_{\mathbf{x}_0})$ \hfill $\triangleright$ Eq.~\eqref{eq:covarianceBLR}
\For{$t = S+1$ to $N$ in batches of $B$}
    \State Draw candidate pool $\mathcal{P} \gets A$ perturbations near $\mathbf{x}_0$
    \For{each $\mathbf{z} \in \mathcal{P}$}
        \State $a(\mathbf{z}) \gets \pi_{\mathbf{x}_0}(\mathbf{z}) \cdot \mathbf{z}^\top \mathbf{V}_{\bmphi}\, \mathbf{z}$ \hfill $\triangleright$ Eq.~\eqref{eq:finalAcq}
    \EndFor
    \State Select the top $B$ candidates: $\mathcal{B} \gets \operatorname{argtop}_{B}\; a(\mathbf{z})$ \hfill $\triangleright$ Greedy selection
    \State Query $f_b$ on $\mathcal{B}$; update $\mathcal{Z} \gets \mathcal{Z} \cup \mathcal{B}$, $\mathbf{y} \gets [\mathbf{y}; f_b(\mathcal{B})]$
    \State Refit surrogate: $\hat{\bmphi}, \mathbf{V}_{\bmphi} \gets \text{BLR}(\mathcal{Z}, \mathbf{y}, \pi_{\mathbf{x}_0})$
\EndFor
\State \Return $\hat{\bmphi}$, $\mathbf{V}_{\bmphi}$
\end{algorithmic}
\end{algorithm}

\begin{lemma}
\label{lem:InfGainRate}
Given Bayesian linear regression surrogate in \eqref{eq:bayesian_surrogate}, the set of perturbations
\( \matZ\) with corresponding predictions, $
\hat y(\mathbf z)$, $\mathbf V_{\bmphi, 0} = \mathbf I \in \mathbb R^{d \times d}$ and posterior covariance after $t$ queries given by \eqref{eq:posteriortqueries}, assuming $\|\mathbf z_t\|_2 \le 1$ and $0 \le \pi_{\mathbf x_0}(\mathbf z_t) \le 1$ for all $t \geq 1$, expected information gain acquisition obtained in \eqref{eq:finalAcq} leads to
\[
\sum_{s=1}^t
\pi_{\mathbf x_0}(\mathbf z_s)\,
\mathbf z_s^\top \mathbf V_{\bmphi, s-1} \mathbf z_s
\;\le\;
2d \log\!\left(1 + \frac{t}{d}\right).
\]
\end{lemma}

\noindent \textbf{Remark}:~From the above lemma, we infer that as $t$ increases, the information matrix $\mathbf V_{\bmphi,t}^{-1}$ grows monotonically (since each summand is positive semidefinite), and consequently the covariance $\mathbf V_{\bmphi,t}$ shrinks, reflecting reduced uncertainty. Since the determinant satisfies $|\mathbf V_{\bmphi,t}|
=
\prod_{i=1}^d \lambda_i(\mathbf V_{\bmphi,t})$, the determinant of the posterior matrix measures the volume of the confidence ellipsoid. Hence $\log |\mathbf V_{\bmphi,0}| - \log |\mathbf V_{\bmphi,t}|
=
\log \frac{|\mathbf V_{\bmphi,0}|}{|\mathbf V_{\bmphi,t}|}$, 
quantifies the total reduction in uncertainty volume. In the sequel, we plot D-efficiency, which is proportional to the uncertainty volume. Note that $\sum_{s=1}^t
\pi_{\mathbf x_0}(\mathbf z_s)\,
\mathbf z_s^\top \mathbf V_{\bmphi, s-1} \mathbf z_s$ represents the contribution of the $t$ samples towards reduction in uncertainty. Further, this bound demonstrates that even when queries are chosen using $\mathcal{A}_{E}$ to maximize information gain, the total accumulated information grows only logarithmically in $t$; this is the diminishing returns phenomenon where early queries substantially reduce uncertainty, while later queries contribute lesser. The multiplicative factor $d$ indicates that the total information gain scales with the number of independent directions. Since $d$ is fixed and as $t \rightarrow \infty$, we have $\sum_{s=1}^t
\pi_{\mathbf x_0}(\mathbf z_s)\,
\mathbf z_s^\top \mathbf V_{\bmphi, s-1} \mathbf z_s
=
\mathcal{O}(d \log t)$. 

\begin{theorem}
\label{thm:gaussian_lm_clean}
Given the source linear model in \eqref{eq:sourceLinModel}, \( \bmphi^\star \in \mathbb R^d \) as the true (unknown) importance vector and  $\mathbf S_t
:=
\sum_{s=1}^t \pi_{\mathbf x_0}(\mathbf z_s)
\mathbf z_s \mathbf z_s^\top$, then for any $\delta \in (0,1)$,
with probability at least $1-\delta$,
\begin{equation}
\|
\hat{\bmphi}_t - \bmphi^\star
\|_{\mathbf V_{\bmphi,t}^{-1}}
\le
\sigma
\sqrt{
d
+
2\sqrt{d \log \tfrac{1}{\delta}}
+
2 \log \tfrac{1}{\delta}
}
+
\|\bmphi^\star\|_{V_{\bmphi,t}}.
\label{eq:lm_final_bound}
\end{equation}
\end{theorem}

\begin{corollary}[Sample Complexity for $\ell_2$-Accuracy]
\label{cor:sample_complexity}
Under the assumptions of Theorem~\ref{thm:gaussian_lm_clean}, suppose that minimum eigenvalue of $\mathbf S_t$ represented as $\lambda_{\min}(\mathbf S_t) \ge \kappa t
\quad \text{for some } \kappa > 0$. Then for any $\nu > 0$ and $\delta \in (0,1)$, with probability at least $1-\delta$, $\|
\hat{\bmphi}_t - \bmphi^\star
\|_2
\le
\nu$ 
whenever
\begin{align}
t
\ge
\frac{1}{\kappa}
\left(
\frac{(\beta_\delta + \|\bmphi^\star\|_2)^2}{\nu^2}
-
1
\right),
\label{eq:samplecomplexityt}
\end{align}
where $\beta_\delta
=
\sigma
\sqrt{
d
+
2\sqrt{d \log \tfrac{1}{\delta}}
+
2 \log \tfrac{1}{\delta}
}$.
\end{corollary}

\noindent \textbf{Sample Complexity:}
From the corollary we obtained the sufficient condition on the number of sample queries $t$ as a function of data and the model performance. In \eqref{eq:samplecomplexityt}, ignoring the negligible constant term $-1$ for large $t$, and noting that $2\sqrt{d \log \tfrac{1}{\delta}}
\le
d + \log \tfrac{1}{\delta}
$, we have $\beta_\delta^2
=
\mathcal O\!\left(d + \log \tfrac{1}{\delta}\right)$. 
If $\|\bmphi^\star\|_2$ is treated as a fixed constant independent of $t$,
then $(\beta_\delta + \|\bmphi^\star\|_2)^2
=
\mathcal O\!\left(d + \log \tfrac{1}{\delta}\right)$. Substituting this into the lower bound on $t$ gives
\[
t
=
\mathcal O
\!\left(
\frac{d + \log(1/\delta)}{\kappa \nu^2}
\right).
\]
The result above states that the required number of queries under $\mathcal{A}_{\text{E}}$ acquisition function grows linearly with the dimension $d$, logarithmically with the confidence parameter $1/\delta$, and quadratically
with the inverse accuracy $1/\nu^2$. Thus, achieving higher precision
demands a quadratic increase in samples, while increasing confidence only
incurs a mild logarithmic penalty. 

\begingroup
\renewcommand{\thefootnote}{}
\footnotetext{For complete proofs please refer: \url{https://arxiv.org/abs/2603.14894}\\
For implementation code please follow this \href{https://github.com/sumedhachugh/EAGLE-Informative-Perturbation-Selection-for-Uncertainty-Aware-Post-hoc-Explanations}{link}}
\endgroup

\noindent \textbf{Relationship to Existing Local Surrogate Methods:}~ Perturbation-based explanation methods primarily differ in how they generate samples in the neighborhood of $\mathbf x_0$, and the explainer they employ. LIME generates perturbations $\mathbf z \sim p(\mathbf z)$ independently from a predefined distribution and fits a locally weighted linear surrogate by minimizing $\sum_{i=1}^{n} \pi_{\mathbf x_0}(\mathbf z_i)\big(f(\mathbf z_i)-\mathbf z_i^\top \mathbf w\big)^2$, 
where $\pi_{\mathbf x_0}(\mathbf z)$ is a locality kernel assigning higher weights to perturbations closer to $\mathbf x_0$. In this formulation, perturbations are sampled independently of their informativeness. GLIME, which generalizes several methods such as ALIME, DLIME, etc, modifies this procedure by absorbing the kernel weight into the sampling distribution; by sampling perturbations from a local distribution proportional to the kernel, $\mathbf z \sim q(\mathbf z) \propto \pi_{x_0}(\mathbf z)$.
While this improves stability, perturbations are still drawn i.i.d.\ from a fixed distribution, i.e., the sampling process does not adapt to the information contained in previously observed perturbations. Uncertainty Sampling LIME (US-LIME) extends the perturbation strategy of LIME and selects samples that are both close to the instance and near the black-box decision boundary using an uncertainty-sampling criterion based on class probabilities. Unlike LIME-based methods that employ linear surrogates, Tilia uses a decision tree surrogate to capture nonlinear feature interactions. However, for all of the above methods, perturbations are still generated from a predefined distribution around $\mathbf x_0$ with heuristic or self-defined notions of uncertainty.

\noindent Another family of explainers is the model-agnostic post-hoc Bayesian explainers, such as BayLIME and BayesLIME, which extend the LIME framework by modeling the importance weights as random variables and performing Bayesian linear regression to obtain posterior uncertainty over these weights. The details of modeling in BayesLIME are as given in Sec.~\ref{sec: problem_setting}. BayesLIME relies on a heuristic, namely the variance of the predictions $\text{var}(\hat y(\mathbf z))$ for subsampling from a randomly sampled pool around $\mathbf x_0$. In particular, given the perturbations in $\mathbf Z$, the locality kernel $\pi_{\mathbf x}(\cdot)$ influences the surrogate model only through the weights assigned to the already observed perturbations via $V_\phi$ and $s^2$. When evaluating a new candidate perturbation $\mathbf z$, BayesLIME  considers the predictive variance which in turn does not include the locality information of the current sample $\mathbf z$, via $\pi_{\mathbf x}(\mathbf z)$. Hence, BayesLIME’s acquisition strategy is not locality-aware when choosing new perturbations. 

\noindent The quantity $\mathbf z^T \mathbf V_{\bmphi} \mathbf z$, common to both BayesLIME and \texttt{EAGLE}, is the squared Euclidean norm of $\mathbf z$ in the geometry induced by $\mathbf V_{\bmphi}$. Intuitively, this quantity measures how strongly the candidate perturbation lies in directions where the posterior covariance of the explanation parameters are large. The variable $s^2$ is a global scaling factor, and does not contain any information regarding the locality. In contrast, \texttt{EAGLE} is a principled approach for which we derive the novel acquisition function as $\arg\max_{\mathbf z}
\pi_{\mathbf x_0}(\mathbf z)\,
\mathbf z^\top
\mathbf V_{\bmphi}
\mathbf z$. Here, the quadratic term $\mathbf z^\top \mathbf V_{\bmphi} \mathbf z$ plays the same role as in BayesLIME. However, the distinguishing factor is the term $\pi_{\mathbf x_0}(\mathbf z)$ which enforces locality around the instance being explained. Intuitively, this rule prioritizes perturbations that both lie close to $\mathbf x_0$ and provide maximal information about uncertain directions of the surrogate parameters. In addition, the principled derivation of \texttt{EAGLE} allows us to establish theoretical guarantees regarding information gain and sample complexity. 

\section{Experiments}
\label{sec:Experiments}

We empirically evaluate the proposed \texttt{EAGLE} framework to evaluate whether the information gain–based acquisition strategy of \texttt{EAGLE} improves perturbation sampling by reducing surrogate uncertainty. We compare against existing perturbation-based explanation methods across tabular and image datasets.

\noindent \textbf{Datasets:}~ We evaluate our \texttt{EAGLE} framework across six datasets spanning tabular and image modalities, chosen to reflect a range of feature dimensionalities, decision boundaries, and application domains.

\noindent \underline{Tabular Datasets:}~We consider four widely used tabular benchmarks, (a) COMPAS is a dataset containing demographic and criminal history features used to predict two-year recidivism risk \cite{COMPAS}, (b) German Credit dataset consists of features encoding financial and personal attributes, used to classify applicants as good or bad credit risks \cite{german}, (c) Adult Income dataset contains records from the $1994$ U.S. Census, with the task of predicting whether an individual's annual income exceeds a certain threshold  based on demographic and employment features \cite{adult}, and (d) Magic (Gamma Telescope) dataset comprises features derived from Monte Carlo simulations of high energy gamma particles, used to distinguish gamma signal events from hadronic background \cite{magic}. These four datasets vary in size, feature count, and class balance, providing a comprehensive benchmark for evaluating explanation stability and faithfulness in the tabular setting.

\noindent \underline{Image Datasets:}~We additionally evaluate \texttt{EAGLE} on two image classification tasks to assess its applicability to high dimensional inputs, where explanations operate over interpretable superpixel regions. The two datasets are as follows: (a) MNIST is a dataset of 28×28 grayscale handwritten digit images \cite{mnist}, and 
(b) ImageNet (224×224 images) is drawn from the ILSVRC-2012 challenge and presents a substantially more challenging setting with high-resolution natural images \cite{ImageNet}. The ImageNet experiments test \texttt{EAGLE}'s scalability to realistic image classification scenarios where the number of superpixel features and the complexity of the decision boundary are significantly higher than in MNIST.

\noindent \textbf{Experimental Setup:}~For all tabular datasets, we train a Random Forest classifier with $100$ estimators as the black-box model, using an $80/20$ train-test split with standard scaling applied to the features. It achieves test accuracy of $83.7\%, 71.5\%, 83.1\%$, and $86.8\%$ on COMPAS, German Credit, Adult Income, and Magic, respectively. For image datasets, we segment each input into superpixel regions using SLIC \cite{slic}. ImageNet images are segmented into $50$ superpixels for natural images, while simpler MNIST images are segmented into $20$ superpixels. For MNIST, we use a Convolutional Neural Network (CNN) with two layers, achieving $99.15\%$ test accuracy. For ImageNet, we use a pretrained VGG-16 model (torchvision), achieving an accuracy of $93\%$.

\noindent The acquisition process begins with $S=10$ seed perturbations, after which samples are selected in batches of $B=10$ from a candidate pool of $A=1000$, up to a total budget of $N=500$ queries. We evaluate 50 test instances (samples of interest) per dataset. 
The Bayesian linear regression surrogate uses prior with precision $\lambda = d$, scaling the regularization strength with the input dimensionality. The kernel width is set to $0.75\sqrt{d}$, where $d$ is the number of input features, following the standard LIME configuration.

\noindent \textbf{Baselines:~}  We compare our approach against several perturbation-based local explanation methods and their variants that differ in perturbation generation strategies and surrogate modeling approaches. In particular, we evaluate our proposed acquisition function within the \texttt{EAGLE} framework and compare it against the following baselines.\\
\textbf{LIME}~\cite{lime}(2016): generates local explanations by sampling perturbations around an instance and fitting a locality-weighted sparse linear surrogate model.\\
\textbf{GLIME}~\cite{GLIME2023}(2023): improves LIME by generating perturbations using a kernel-based sampling strategy that preserves feature dependencies in the data distribution. \\
\textbf{BayLIME}~\cite{zhao2021baylime}(2021) extends LIME by employing a Bayesian linear model and prior for  uncertainty in feature importance estimates. \\
\textbf{DLIME}~\cite{dlime}(2021) improves explanation stability by generating perturbations from clusters learned from the dataset rather than random sampling. \\
\textbf{US-LIME}~\cite{uslime}(2024) models uncertainty in the perturbation process to produce more reliable local explanations. \\
\textbf{Tilia}~\cite{tilia}(2025) replaces the linear surrogate with a shallow decision tree to better capture nonlinear local decision boundaries. \\
\textbf{BayesLIME}~\cite{slack2021reliable}(2021), also called Focus Sampling selects perturbations that are most informative for improving the Bayesian surrogate model. \\
\textbf{UnRAvEL}~\cite{unravel}(2022) selects perturbations using GP based BO to maximize the relevance of sampled instances for local surrogate training.

\subsection{Comparison with Baselines}

\textbf{Explanation Stability:}~ A stable explanation method should produce consistent feature attributions across independent runs on the same instance. To evaluate this, we measure the pairwise overlap between the top-$5$ features selected across repeated explanations using \textbf{Jaccard similarity}. Given two feature sets $S_i$ and $S_j$ of size $k$,
\[
\text{Jaccard}(i,j) = \frac{|S_i \cap S_j|}{|S_i \cup S_j|},
\]
averaged over all pairs of runs and all test instances. A Jaccard score of $1$ indicates that every run identifies the same top features, while lower values indicate instability in the explanation. In Table~\ref{tab:jaccard5}, we compare the stability performance of \texttt{EAGLE} against the baseline explanation methods across tabular and image datasets. \texttt{EAGLE} consistently achieves the highest or near-highest stability across most datasets, demonstrating the effectiveness of the proposed strategy. On tabular datasets, \texttt{EAGLE} significantly outperforms existing methods such as LIME, BayLIME, and GLIME. On image datasets, \texttt{Eagle} achieves the strongest stability on MNIST and the second best on ImageNet. While some baselines such as Tilia and US-LIME achieve competitive results on specific datasets, their performance varies considerably across tasks, indicating limited robustness. In contrast, \texttt{EAGLE} maintains consistently strong performance across both tabular and image domains. Note that DLIME is not evaluated on image datasets because its clustering-based neighborhood construction does not directly extend to high-dimensional image inputs.

\begin{table}[t]
\centering
\small
\setlength{\tabcolsep}{3.5pt}

\caption{Explanation stability (Jaccard Similarity $\uparrow$). Mean over 50 instances (standard deviation in subscript, best in \textbf{bold}, second-best \underline{underlined}).}
\label{tab:jaccard5}
\resizebox{\columnwidth}{!}{%
\begin{tabular}{lcccccc}
\toprule

\multirow{2}{*}{Method} 
& \multicolumn{4}{c}{Tabular Datasets} 
& \multicolumn{2}{c}{Image Datasets} \\

\cmidrule(lr){2-5} \cmidrule(lr){6-7}

& \textsc{COMPAS} & \textsc{GERMAN} & \textsc{ADULT} & \textsc{MAGIC} & \textsc{MNIST} & \textsc{ImageNet} \\

\midrule

LIME 
& $0.772_{\pm 0.110}$ 
& $0.645_{\pm 0.108}$ 
& $0.669_{\pm 0.045}$
& $0.647_{\pm 0.092}$
& $0.733_{\pm 0.152}$ & $0.704_{\pm 0.137}$ \\

GLIME 
& $0.635_{\pm 0.082}$ 
& $0.546_{\pm 0.111}$ 
& $0.559_{\pm 0.107}$ 
& $0.624_{\pm 0.080}$
& $0.519_{\pm 0.290}$ & $\mathbf{0.920_{\pm 0.08}}$ \\

Tilia 
&$\mathbf{0.834_{\pm 0.077}}$

& \underline{$0.721_{\pm 0.068}$ }
& \underline{$0.783_{\pm 0.045}$ }
& \underline{$0.647_{\pm 0.092}$}
& $0.733_{\pm 0.178}$ & $0.734_{\pm 0.142}$ \\

US-LIME 
& $0.525_{\pm 0.077}$
& $0.400_{\pm 0.150}$ 
& $0.419_{\pm 0.106}$ 
& $0.587_{\pm 0.118}$ 
& \underline{$0.749_{\pm 0.151}$} & $0.760_{\pm0.148 }$ \\

UnRAvEL 
&$0.501_{\pm 0.055}$ 
& $0.332_{\pm 0.121}$ 
& $0.276_{\pm 0.084}$ 
& $0.427_{\pm 0.051}$
& $0.148_{\pm 0.035}$ & $0.099_{\pm0.034}$ \\

DLIME 
& $0.409_{\pm 0.033}$ 
& $0.121_{\pm 0.062}$ 
& $0.155_{\pm 0.024}$ 
& $0.352_{\pm 0.026}$ 
& -- & -- \\

BayesLIME
&$0.770_{\pm 0.097}$ 
& $0.631_{\pm 0.074}$ 
& $0.674_{\pm 0.110}$ 
& $0.617_{\pm 0.096}$
& $0.765{\pm 0.130}$
&  $0.755{\pm 0.015}$\\

BayLIME 
&$0.779_{\pm 0.105}$ 
& $0.678_{\pm 0.104}$
& $0.663_{\pm 0.045}$ 
& $0.648_{\pm 0.093}$
& $0.720_{\pm 0.155}$ & $0.739_{\pm 0.187}$ \\

\textbf{\texttt{EAGLE}(Ours)} 
& $\underline{0.802_{\pm 0.118}}$
& $\mathbf{0.775_{\pm 0.132}}$
& $\mathbf{0.822_{\pm 0.151}}$
& $\mathbf{0.785_{\pm 0.082}}$
& $\mathbf{0.861_{\pm 0.103}}$ & \underline{$0.825_{\pm 0.107}$}  \\

\bottomrule
\end{tabular}
}
\end{table}

\begin{figure}[t]
    \centering

    \begin{minipage}[t]{0.35\textwidth}
        \centering
        \includegraphics[width=\linewidth]{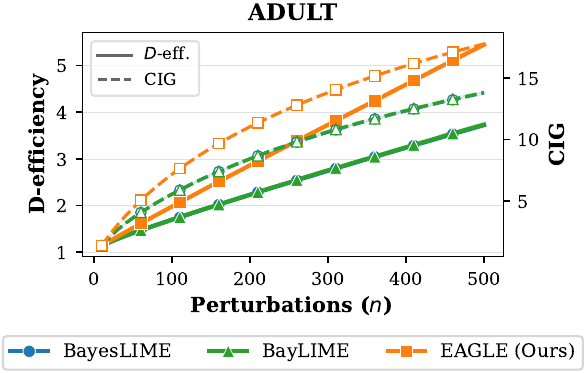}
    \end{minipage}
    \begin{minipage}[t]{0.63\textwidth}
        \centering
        \includegraphics[width=\linewidth]{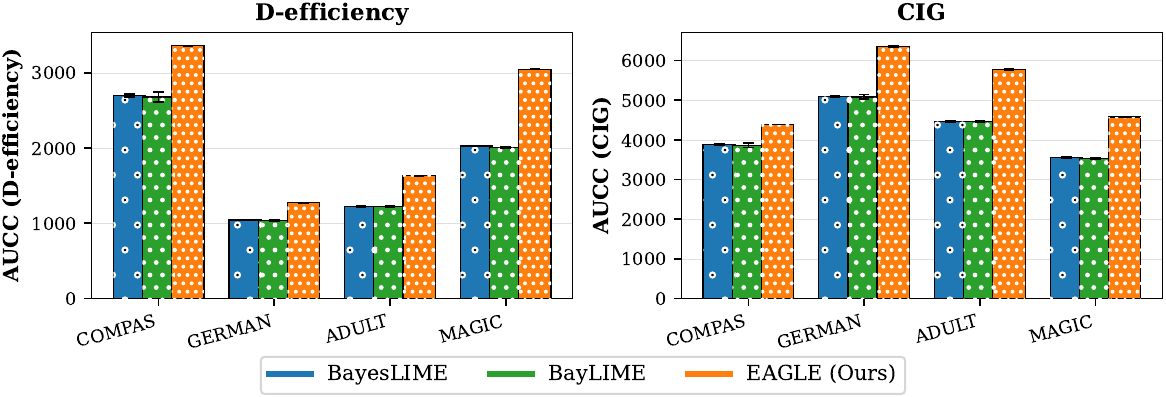}
    \end{minipage}

    \caption{Active sampling convergence behaviour: \texttt{EAGLE} consistently improves both $D$-efficiency and cumulative information gain, leading to higher AUCC across datasets.}
    \label{fig:convergence_aucc}
\end{figure}
\noindent \textbf{Evaluation of Sampling Quality:~} We assess how efficiently each perturbation strategy extracts information from the black-box model, comparing \texttt{EAGLE}'s acquisition functions against existing Bayesian post-hoc explanation methods using two complementary metrics. \textbf{D-efficiency} defined as $\left(\frac{|V_0|}{|V_t|}\right)^{1/d}$, measures the reduction in volume of the posterior covariance ellipsoid relative to the prior \cite{atkinson2007}. This aligns with theory (Lemma \ref{lem:InfGainRate}), as it measures the amount of information gain as query samples are accrued. Higher D-efficiency indicates that the selected perturbations shrink uncertainty more uniformly across all surrogate coefficient directions. This metric is invariant under reparametrization, making it a robust choice for comparing acquisition strategies. While D-efficiency captures the geometric reduction in posterior uncertainty, it does not reflect the absolute amount of information gained. We complement this metric with cumulative information gain (CIG) which quantifies the total differential entropy reduction over the perturbation budget \cite{chaloner1995}. In the leftmost plot in Fig.~{\ref{fig:convergence_aucc}}, we observe that D-efficiency justifies the logarithmic behavior (across $t$), as predicted by Lemma~\ref{lem:InfGainRate}. \texttt{EAGLE} exhibits a steeper convergence trajectory from the earliest acquisition steps, attaining a D-efficiency approximately $1.5\times$
that of BayesLIME and BayLIME at $n=500$. This faster growth reflects that perturbations selected by \texttt{EAGLE} reduce the posterior covariance more efficiently per query alongside per-step entropy as confirmed by the CIG curve (plots for all datasets are provided in the supplementary material). 

\subsection{Ablation Study}
\label{Ablation Study}

In this section, we investigate the sensitivity of \texttt{EAGLE} to key design choices. We first evaluate sample efficiency by measuring the query budget at which \texttt{EAGLE} matches the performance of BayesLIME (\cite{slack2021reliable}), using it as a baseline since both methods share the same BLR surrogate, isolating the effect of the acquisition strategy. We then ablate on prior precision, candidate pool size, and superpixel granularity. We use D-efficiency and the Comprehensive Consistency Metric (CCM) \cite{BELIEF2025}, defined as
\[
\text{CCM} = (1 - \text{ASFE}) \times \text{ARS},
\]  
which combines sign-flip entropy, given as ASFE (Average Sign Flip Entropy) and rank similarity given as ARS (Average Rank Similarity) across independent runs. A CCM close to 1 indicates stable explanations under a given configuration. 

\noindent \textbf{Sample Efficiency:}~
We quantify the sample efficiency of \texttt{EAGLE} by measuring the crossover budget, the number of queries at which \texttt{EAGLE}  first matches the D-efficiency and CCM that Focused Sampling attains at its full budget of $N=500$. Table~\ref{tab:sample-efficiency} reports the sample efficiency of the proposed approach relative to BayesLIME across six datasets under two complementary metrics. For D-efficiency, \texttt{EAGLE}  matches or exceeds BayesLIME@500 on every test instance across all datasets, requiring only 310-390 queries to achieve equivalent design quality, yielding savings of 22\% to 38\%. For CCM, crossover occurs in the large majority of instances, with savings ranging from  52\% to 88\%, this confirms how aggressively our method concentrates the posterior, in much lower run times. 

\begin{table}[ht]
\centering
\setlength{\tabcolsep}{5pt}
\resizebox{\columnwidth}{!}{%
\begin{tabular}{@{}l r r r @{\hspace{8pt}} r r r @{\hspace{12pt}} r r@{}}
\toprule
& \multicolumn{3}{c}{\textbf{D-efficiency $\uparrow$}} 
& \multicolumn{3}{c}{\textbf{CCM} $\uparrow$}
& \multicolumn{2}{c}{\textbf{Runtime(s)}$\downarrow$} \\
\cmidrule(lr){2-4} \cmidrule(lr){5-7} \cmidrule(lr){8-9}
\textbf{Dataset} & {BayesLIME@500} & \textbf{\texttt{EAGLE}@} & {Saving} 
                 & {BayesLIME@500} & \textbf{\texttt{EAGLE}@} & {Saving}
                 & {BayesLIME} & \textbf{\texttt{EAGLE}} \\
\midrule

\textsc{COMPAS}  & 9.623 & 390 & 22\% 
                 & 0.604  & 70  & 86\%
                 & 14.56 & 8.16 \\[3pt]

\textsc{GERMAN}  & 3.705 & 340 & 32\% 
                 & 0.653  & 180 & 64\%
                 & 23.63 & 15.61 \\[3pt]

\textsc{ADULT}   & 3.724 & 310 & 38\% 
                 & 0.662  & 240 & 52\%
                 & 17.27 & 10.46 \\[3pt]

\textsc{MAGIC}   & 7.050 & 310 & 38\% 
                 & 0.579  & 110 & 78\%
                 & 31.29 & 26.72 \\[3pt]

\textsc{MNIST}   & 3.094 & 380 & 24\% 
                 & 0.648  & 60  & 88\%
                 & 3.36 & 0.69 \\[3pt]

\textsc{ImageNet} & 2.448 & 390 & 22\%  
                  & 0.816  & 200 & 60\%
                  & 34.73 & 7.08 \\[3pt]

\bottomrule
\end{tabular}
}
\vspace{4pt}

\caption{Query savings achieved by \texttt{EAGLE} over BayesLIME across tabular and image datasets.
BayesLIME@500 reports the value attained by BayesLIME at $n{=}500$ queries; \texttt{EAGLE}@ gives the budget
at which \texttt{EAGLE} reaches the same quality. The final columns report wall-clock runtime per
explanation at $n{=}500$ queries.}

\label{tab:sample-efficiency}

\end{table}

\noindent \textbf{Hyperparameter Sensitivity:}~ We ablate three key hyperparameters of \texttt{EAGLE}: (1) the prior precision $\lambda$, (2) the candidate pool size $\mathcal{P}$, and (3) the number of superpixel segments $d$. The prior precision controls regularization strength in the Bayesian linear surrogate, directly influencing the posterior covariance $V_{\phi}$ from which acquisition scores are derived. For tabular datasets, we vary $\lambda/d \in \{0.01, 0.1, 1, 10, 100\}$ and  pool size $\mathcal{P} \in \{100, 250, 500, 1000, 2000\}$, holding all other parameters at their default values. For image datasets, the number of superpixel segments determines the resolution at which \texttt{EAGLE} explains the black-box prediction; fewer segments yield coarse, region-level attributions, while more segments produce finer explanations at the cost of a high dimensional surrogate. We vary the number of SLIC segments over $\{10, 15, 20,  30, 40, 45\}$ with a fixed perturbation budget of $N = 500$.  Figure~\ref{fig:Ablation}(a) shows that D-efficiency increases monotonically with $\lambda$ across all tabular datasets, as stronger prior regularization concentrates the posterior. CCM improves up to $\lambda/d \approx 10$, where regularization stabilizes feature rankings, then declines at $\lambda/d = 100$ for COMPAS, Adult Income and Magic as over-regularization collapses coefficient magnitudes, reducing meaningful differentiation between features. The default setting $\lambda = d$ provides a favorable trade-off between posterior concentration and explanation consistency. In Figure~\ref{fig:Ablation}(b), both D-efficiency and CCM remain stable across all candidate pool sizes; this robustness shows that \texttt{EAGLE} does not require large candidate pools to achieve good performance, reducing computational overhead. 
Figure~\ref{fig:Ablation}(c) shows that D-efficiency decreases with superpixel count on both image datasets, consistent with the $\mathcal{O}(d \log t)$ scaling established in Lemma~\ref{lem:InfGainRate}. CCM for ImageNet stabilizes at higher superpixel counts ($d > 20$)  as the underlying images are visually complex, giving the surrogate enough structure to learn consistent attributions. In contrast, MNIST CCM degrades beyond $d = 20$ since the small $28 \times 28$ images lack sufficient structure for fine-grained superpixels, leading to noisy attributions.

\begin{figure}[t]
    \centering
    \includegraphics[width=\textwidth]{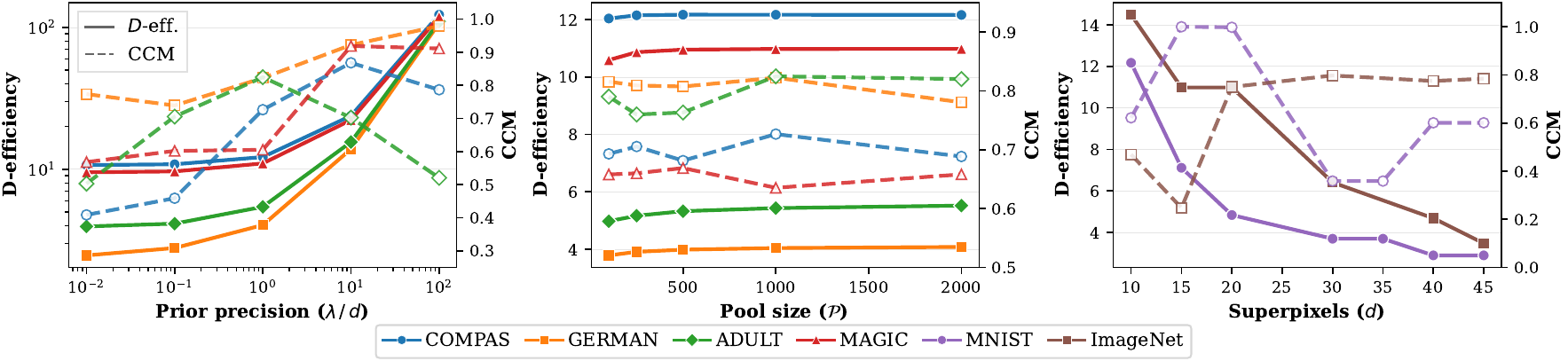}      
    \caption{Hyperparameter sensitivity of \texttt{EAGLE}, reporting D-efficiency (solid) and CCM (dashed). \textbf{(a)}~Prior precision $\lambda/d$: the default $\lambda{=}d$ balances posterior concentration and explanation consistency. \textbf{(b)}~Pool size $|\mathcal{P}|$: \texttt{EAGLE} maintains strong performance across all pool sizes, requiring no careful tuning. \textbf{(c)}~Superpixel count: \texttt{EAGLE} adapts effectively to increasing dimensionality on both MNIST and ImageNet.}
    \label{fig:Ablation}    
\end{figure}

\noindent \textbf{Runtime comparison:}~ In Table~\ref{table:runtime}, we  compare the runtime of different explanation methods as the sampling budget increases. LIME-based methods are the fastest due to their simple perturbation and linear surrogate framework. However, their explanations are often unstable and sensitive to sampling variations, making them less reliable despite the low computational cost. In contrast, \texttt{EAGLE} improves explanation reliability while maintaining competitive runtime. Compared to BayesLIME, \texttt{EAGLE} consistently achieves lower runtime across all budgets. For example, at $N{=}500$, \texttt{EAGLE} requires $8.16$ seconds on average, while BayesLIME takes $14.56$ seconds. These results show that the proposed sampling strategy improves efficiency while producing more reliable explanations with only a modest increase in computational cost. Experiments run on AMD Ryzen Threadripper 3960X (24-core, 48-thread) CPU. Each cell reports mean$\pm$std over 5 instances $\times$ 5 repeats. 
\begin{table}[t]
\centering

\setlength{\tabcolsep}{5pt}
\renewcommand{\arraystretch}{1.15}
\resizebox{\columnwidth}{!}{%
\begin{tabular}{l l r r r r}
\toprule
\multirow{2}{*}{\textbf{Method}} &
\multirow{2}{*}{\textbf{Category}} &
\multicolumn{4}{c}{\textbf{Runtime(s) on \textsc{COMPAS}}} \\
\cmidrule(lr){3-6}
 & & $N{=}50$ & $N{=}100$ & $N{=}200$ & $N{=}500$ \\
\midrule
LIME
  & LIME-based
  & ${0.028} \pm 0.001$
  & ${0.038} \pm 0.001$
  & ${0.056} \pm 0.003$
  & ${0.113} \pm 0.003$ \\

BayLIME
  & LIME-based
  & $0.029 \pm 0.002$
  & $0.039 \pm 0.001$
  & $0.058 \pm 0.003$
  & $0.139 \pm 0.022$ \\

GLIME
  & LIME-based
  & $0.032 \pm 0.001$
  & $0.045 \pm 0.004$
  & $0.079 \pm 0.004$
  & $0.160 \pm 0.012$ \\

DLIME
 & LIME-based
  & $0.020 \pm 0.003$
  & $0.023 \pm 0.003$
  & $0.029 \pm 0.004$
  & $0.044 \pm 0.004$ \\
Tilia
  & LIME-based
  & $0.041 \pm 0.001$
  & $0.061 \pm 0.004$
  & $0.095 \pm 0.010$
  & $0.241 \pm 0.017$ \\

US-LIME
  & LIME-based
  & $0.234 \pm 0.025$
  & $0.389 \pm 0.031$
  & $0.714 \pm 0.055$
  & $1.653 \pm 0.106$ \\

\midrule
BayesLIME
  & Bayesian Active
  & $1.047 \pm 0.149$
  & $2.520 \pm 0.218$
  & $5.189 \pm 0.415$
  & $14.562 \pm 0.838$ \\

\texttt{EAGLE} (ours)
  & Bayesian Active
  & $0.651 \pm 0.059$
  & $1.386 \pm 0.069$
  & $3.195 \pm 0.095$
  & $8.160 \pm 0.188$ \\

\midrule
UnRAvEL
  & BO+GP
  & $68.420 \pm 23.941$
  & $73.026 \pm 28.844$
  & $85.330 \pm 38.904$
  & $102.286 \pm 10.256$ \\

\bottomrule
\end{tabular}%
}
\vspace{1pt}
\caption{%
  Wall-clock runtime (seconds, mean\,$\pm$\,std) per explanation on the
  \textsc{COMPAS} dataset across sample budgets $N \in \{50,100,200,500\}$.
}
\label{table:runtime}
\end{table}

\section{Conclusions}
\vspace{-1em}
We introduce \texttt{EAGLE}, an active learning framework that selects informative perturbations using information-theoretic acquisition functions and estimates confidence in feature attributions through Bayesian surrogate modeling. We provide theoretical bounds on the information gain and establish high-probability guarantees on the estimation error of explanation weights, leading to sample complexity bounds. Empirically, \texttt{EAGLE} consistently improves explanation stability and sampling efficiency across datasets, achieving stronger Jaccard similarity and faster information gain convergence compared to existing perturbation-based baselines. These results mark a meaningful step toward the development of mathematically grounded active learning methods for reliable post hoc explainability of black box models.


\bibliographystyle{splncs04}
\bibliography{refs}

\newpage
{\Large \textbf{Supplementary Material}}
\setcounter{section}{0}
\section{Additional Theoretical Results}

\noindent \textbf{Theorem 1 and Proof:}~\emph{Consider a Bayesian linear surrogate model given in \eqref{eq:bayesian_surrogate}
where the posterior covariance is given by \eqref{eq:covarianceBLR}.
Let the \texttt{EAGLE}-based acquisition function be defined as the marginal expected
reduction in entropy of the surrogate parameters \( \bmphi \) obtained by
querying a candidate perturbation \( \mathbf z \),
\[
\mathcal{A}_{\mathrm{E}}(\mathbf z)
=
\mathbb{E}_{y \mid \mathbf z, \mathcal{Z}}
\Big[
\mathrm{H}\big(p(\bmphi \mid \mathcal{D})\big)
-
\mathrm{H}\big(p(\bmphi \mid \mathcal{Z} \cup \{(\mathbf z)\})\big)
\Big].
\]
Then, under single-step greedy acquisition, maximizing
\( \mathcal{A}_{\mathrm{E}}(\mathbf z) \) over a candidate pool is equivalent,
up to additive and multiplicative constants independent of \( \mathbf z \), to
maximizing the predictive variance,
\[
\arg\max_{\mathbf z} \mathcal{A}_{\mathrm{E}}(\mathbf z)
=
\arg\max_{\mathbf z} \; \pi_{\mathbf x_0}(\mathbf z)\ \mathbf z^\top \mathbf V_{\bm\phi} \mathbf z.
\]}
\begin{proof}
Augmenting the current perturbation set \( \mathbf Z \) with a single
candidate perturbation \( \mathbf z \), yielding the augmented set
\( \mathbf Z^{+} = \mathbf Z \cup \{\mathbf z\} \).
Substituting the definition of the acquisition function yields
\begin{align}
\mathcal A_{\mathrm{EIG}}(\mathbf z)
&=
\mathbb{E}_{y \mid \mathbf z, \mathbf Z}
\Big[
\mathrm{H}\big(p(\bmphi \mid \mathbf Z, \mathbf y)\big)
-
\mathrm{H}\big(p(\bmphi \mid \mathbf Z^{+}, \mathbf y^{+})\big)
\Big],
\end{align}
Under the linear--Gaussian surrogate model, the posterior covariance update
depends only on the queried perturbation \( \mathbf z \) and not on the predictions \( y \), which allows the expectation over \( y \) to be restated as:
\[
\mathbb{E}_{y \mid \mathbf z, \mathbf Z}
\Big[
\mathrm{H}\big(p(\bmphi \mid \mathbf Z^{+}, \mathbf y^{+})\big)
\Big]
=
\mathrm{H}\big(p(\bmphi \mid \mathbf Z^{+})\big).
\]

Using the result in Thm.~\ref{th:A_eig}, the information gain for a single perturbation simplifies to
\begin{align}
\mathcal A_{\mathrm{E}}(\mathbf z)
=
\frac{1}{2}
\log \lvert \mathbf V_{\bm\phi} \rvert
-
\frac{1}{2}
\log \lvert \mathbf V_{\bm\phi}^{+}(\mathbf z) \rvert,
\end{align}
expressing the acquisition function as the reduction in log-volume of the posterior covariance.
In the weighted surrogate formulation, each perturbation \( \mathbf z \) is
associated with a locality weight \( \pi_{\mathbf x_0}(\mathbf z) \).
Consequently, the posterior precision update after querying a new perturbation
is given by the weighted update
\begin{equation}
\mathbf V_{\bm\phi}^{-1}
+
\pi_{\mathbf x_0}(\mathbf z)\,
\mathbf z \mathbf z^\top.
\label{eq:weighted_precision}
\end{equation}

Accordingly, the posterior covariance after querying \( \mathbf z \) is
\begin{equation}
\mathbf V_{\bm\phi}^{+}(\mathbf z)
=
\left(
\mathbf V_{\bm\phi}^{-1}
+
\pi_{\mathbf x_0}(\mathbf z)\,
\mathbf z \mathbf z^\top
\right)^{-1}.
\label{eq:weighted_covariance}
\end{equation}

Using the identity \( \lvert \mathbf A^{-1} \rvert = \lvert \mathbf A \rvert^{-1} \),
we obtain
\begin{equation}
\lvert \mathbf V_{\bm\phi}^{+}(\mathbf z) \rvert
=
\left\lvert
\mathbf V_{\bm\phi}^{-1}
+
\pi_{\mathbf x_0}(\mathbf z)\,
\mathbf z \mathbf z^\top
\right\rvert^{-1}.
\label{eq:det_step1}
\end{equation}

The matrix determinant lemma states that for an invertible matrix
\( \mathbf A \) and vectors \( \mathbf u, \mathbf v \),
\[
\lvert \mathbf A + \mathbf u \mathbf v^\top \rvert
=
\lvert \mathbf A \rvert
\bigl( 1 + \mathbf v^\top \mathbf A^{-1} \mathbf u \bigr).
\]
Using the above identity in \eqref{eq:det_step1} yields
\begin{equation}
\left\lvert
\mathbf V_{\bm\phi}^{-1}
+
\pi_{\mathbf x_0}(\mathbf z)\,
\mathbf z \mathbf z^\top
\right\rvert
=
\lvert \mathbf V_{\bm\phi}^{-1} \rvert
\left(
1
+
\pi_{\mathbf x_0}(\mathbf z)\,
\mathbf z^\top \mathbf V_{\bm\phi} \mathbf z
\right).
\label{eq:det_step2}
\end{equation}

Substituting~\eqref{eq:det_step2} into~\eqref{eq:det_step1}
we obtain
\begin{equation}
\lvert \mathbf V_{\bm\phi}^{+}(\mathbf z) \rvert
=
\frac{
\lvert \mathbf V_{\bm\phi} \rvert
}{
1
+
\pi_{\mathbf x_0}(\mathbf z)\,
\mathbf z^\top \mathbf V_{\bm\phi} \mathbf z
}.
\label{eq:weighted_det_final}
\end{equation}

Substituting~\eqref{eq:weighted_det_final} into the expression for acquisition function leads to the weighted expected information gain acquisition function given as
\begin{equation}
\mathcal A_{\mathrm{E}}(\mathbf z)
=
\frac{1}{2}
\log
\left(
1
+
\pi_{\mathbf x_0}(\mathbf z)\,
\mathbf z^\top \mathbf V_{\bm\phi} \mathbf z
\right).
\label{eq:weighted_eig}
\end{equation}
Since the logarithm is a strictly monotonically increasing function, the ordering
of candidate perturbations induced by the EIG acquisition function is preserved
by its argument. Using~\eqref{eq:weighted_eig}, the EIG acquisition function is therefore
equivalently written as

\begin{equation}
    \arg\max_{z} A_E(z)
=
\arg\max_{z} \frac{1}{2}
\log\left(1+\pi_{x_0}(z)\, z^\top V_{\phi} z\right)
=
\arg\max_{z} \pi_{x_0}(z)\, z^\top V_{\phi} z .
\label{eq:eig_argmax}
\end{equation}

\hfill $\square$
\end{proof}

\noindent\textbf{Lemma 1 and proof}~\emph{
Consider a Bayesian linear regression surrogate in \eqref{eq:bayesian_surrogate}, the set of perturbations
\( \matZ\) with corresponding predictions, $
\hat y(\mathbf z)$, $\mathbf V_{\bmphi, 0} = \mathbf I \in \mathbb R^{d \times d}$ and posterior covariance after $t$ queries given by \eqref{eq:posteriortqueries}. Assuming $\|\mathbf z_t\|_2 \le 1$ and $0 \le \pi_{\mathbf x_0}(\mathbf z_t) \le 1$ for all $t \geq 1$, expected information gain based acquisition as obtained in \eqref{eq:finalAcq} leads to
\[
\sum_{s=1}^t
\pi_{\mathbf x_0}(\mathbf z_s)\,
\mathbf z_s^\top \mathbf V_{\bmphi, s-1} \mathbf z_s
\;\le\;
2d \log\!\left(1 + \frac{t}{d}\right).
\]
}

\begin{proof}
From \eqref{eq:finalAcq}, in the $t$-th step the perturbations are selected greedily by maximizing weighted expected information gain as
\[
\mathbf z_t
=
\arg\max_{\mathbf z}
\;
\pi_{\mathbf x_0}(\mathbf z)\,
\mathbf z^\top \mathbf V_{\bmphi, t-1} \mathbf z.
\]
Using the identity $\mathbf V^{-1}_{\bmphi, t}
=
\left(
\mathbf V^{-1}_{\bmphi, t-1}
+
\pi_{\mathbf x_0}(\mathbf z_t)\,
\mathbf z_t \mathbf z_t^\top
\right)$, and using the matrix determinant lemma, we obtain the following:
\begin{align}
|\mathbf V_{\bmphi, t}|
=
\frac{
|\mathbf V_{\bmphi, t-1}|
}{
1
+
\pi_{\mathbf x_0}(\mathbf z_t)
\mathbf z_t^\top \mathbf V_{\bmphi, t-1} \mathbf z_t
},
\end{align}
where $|\cdot|$ represents the determinant operator. 
Taking logarithm on both sides and rearranging
\[
\log |\mathbf V_{\bmphi, t-1}|
-
\log |\mathbf V_{\bmphi, t}|
=
\log
\left(
1
+
\pi_{\mathbf x_0}(\mathbf z_t)
\mathbf z_t^\top \mathbf V_{\bmphi, t-1} \mathbf z_t
\right).
\]
Telescoping from \(s=1\) to \(t\),
\[
\log |\mathbf V_{\bmphi, 0}|
-
\log |\mathbf V_{\bmphi, t}|
=
\sum_{s=1}^t
\log
\left(
1
+
\pi_{\mathbf x_0}(\mathbf z_s)
\mathbf z_s^\top \mathbf V_{\bmphi, s-1} \mathbf z_s
\right).
\]
Using the inequality
\(
\log(1+x) \ge \frac{x}{1+x}
\)
for \(x \ge 0\), and the fact that
\(
\mathbf z_s^\top \mathbf V_{\bmphi, s-1} \mathbf z_s \le 1
\)
since \(\|\mathbf z_s\|_2 \le 1\) and \(\mathbf V_{\bmphi, s-1} \preceq \mathbf I\),
we obtain
\begin{align}
\log |\mathbf V_{\bmphi, 0}|
-
\log |\mathbf V_{\bmphi, t}|
\;\ge\;
\frac{1}{2}
\sum_{s=1}^t
\pi_{\mathbf x_0}(\mathbf z_s)
\mathbf z_s^\top \mathbf V_{\bmphi, s-1} \mathbf z_s.
\label{eq:diff_cov}
\end{align}

Since $\mathbf V_{\bmphi, 0} = \mathbf I$, $|\mathbf V_{\bmphi, 0}|=1$ and hence $\log |\mathbf V_{\bmphi, 0}| = 0$. This leads to $\log |\mathbf V_{\bmphi, 0}|
-
\log |\mathbf V_{\bmphi, t}|
=
\log |\mathbf V_{\bmphi, t}^{-1}|$.

\noindent Let $\lambda_1,\dots,\lambda_d$ be the eigenvalues of $\mathbf V_{\bmphi, t}^{-1}$. Using $|\mathbf V_{\bmphi, t}^{-1}|
=
\prod_{i=1}^d \lambda_i$ and the AM-GM inequality, we have that 
\begin{align}
|\mathbf V_{\bmphi, t}^{-1}|
= \prod_{i=1}^d \lambda_i
\le
\left(
\frac{1}{d}\sum_{i=1}^d \lambda_i
\right)^d
=
\left(
\frac{\operatorname{Tr}(\mathbf V_{\bmphi, t}^{-1})}{d}
\right)^d.
\end{align}

Taking logarithms,
\begin{equation}
\label{eq:bound_logcovt}
\log |\mathbf V_{\bmphi, t}^{-1}|
\le
d
\log
\left(
\frac{\operatorname{Tr}(\mathbf V_{\bmphi, t}^{-1})}{d}
\right).
\end{equation}

We now bound the trace of the posterior precision matrix in order to arrive at the final result. Since $\mathbf V_{\bmphi,t}^{-1}
=
\mathbf I
+
\sum_{s=1}^t
\pi_{\mathbf x_0}(\mathbf z_s)\mathbf z_s \mathbf z_s^\top$, 
taking traces gives
\begin{align}
\operatorname{Tr}(\mathbf V_{\bmphi, t}^{-1})
=
d
+
\sum_{s=1}^t
\pi_{\mathbf x_0}(\mathbf z_s)
\|\mathbf z_s\|_2^2.
\end{align}

Since $\|\mathbf z_s\|_2 \le 1$ and $\pi_{\mathbf x_0}(\mathbf z_s)\le 1$, we have $\operatorname{Tr}(\mathbf V_{\bmphi, t}^{-1})
\le
d + t$. Substituting into \eqref{eq:bound_logcovt}, we obtain
\begin{align}
-\log |\mathbf{V}_{\bmphi, t}| = \log |\mathbf V_{\bmphi, t}^{-1}|
\le
d
\log
\left(
\frac{d+t}{d}
\right)
=
d
\log
\left(
1 + \frac{t}{d}
\right).
\end{align}

Substituting the above in \eqref{eq:diff_cov}, we obtain
\begin{align}
\frac{1}{2}
\sum_{s=1}^t
\pi_{\mathbf x_0}(\mathbf z_s)\,
\mathbf z_s^\top \mathbf V_{\bmphi, s-1} \mathbf z_s \leq \log |\mathbf V_{\bmphi, 0}|
-
\log |\mathbf V_{\bmphi, t}|
\le
d
\log
\left(
1 + \frac{t}{d}
\right),
\end{align}
and hence, $\sum_{s=1}^t
\pi_{\mathbf x_0}(\mathbf z_s)\,
\mathbf z_s^\top \mathbf V_{\bmphi, s-1} \mathbf z_s
\;\le\;
2d
\log
\left(
1 + \frac{t}{d}
\right)$. 

As $t \rightarrow \infty$ implies  $\log\!\left(1+\frac{t}{d}\right) = O(\log t)$, we conclude
\[
\sum_{s=1}^t
\pi_{\mathbf x_0}(\mathbf z_s)\,
\mathbf z_s^\top \mathbf V_{s-1} \mathbf z_s
=
O(d \log t).
\]
\hfill $\square$
\end{proof}

\noindent\textbf{Theorem 2 and Proof}
\emph{
Consider the source linear model in \eqref{eq:sourceLinModel}, \( \bmphi^\star \in \mathbb R^d \) denote the true (unknown) importance vector and let
\[
\mathbf S_t
=
\sum_{s=1}^t \pi_{\mathbf x_0}(\mathbf z_s)
\mathbf z_s \mathbf z_s^\top,
\qquad
\mathbf V_{\bmphi,t}
=
(\mathbf I + \mathbf S_t)^{-1},
\]
and the posterior mean after $t$ samples is given as $\hat{\bmphi}_t
=
\mathbf V_{\bmphi,t}
\sum_{s=1}^t
\mathbf z_s y_s$. Then for any $\delta \in (0,1)$,
with probability at least $1-\delta$,
\begin{equation}
\|
\hat{\bmphi}_t - \bmphi^\star
\|_{\mathbf V_{\bmphi,t}^{-1}}
\le
\sigma
\sqrt{
d
+
2\sqrt{d \log \tfrac{1}{\delta}}
+
2 \log \tfrac{1}{\delta}
}
+
\|\bmphi^\star\|_{V_{\bmphi,t}}.
\label{eq:lm_final_bound}
\end{equation}
}

\begin{proof}
Given $\mathbf S_t$, the posterior mean after $t$ samples is given as $\hat{\bmphi}_t
=
\mathbf V_{\bmphi,t}
\sum_{s=1}^t
\mathbf z_s y_s$. First, we define the noise-weighted design vector $\boldsymbol\xi_t
=
\sum_{s=1}^t
\sqrt{\pi_{\mathbf x_0}(\mathbf z_s)}\varepsilon_s \mathbf z_s$. Since the $\varepsilon_s$ are independent Gaussian variables, we have that 
\[
\boldsymbol\xi_t
\sim
\mathcal N
\left(
\mathbf 0,
\sigma^2 \mathbf S_t
\right).
\]
From the definition of the estimator, we have
\[
\hat{\bmphi}_t
=
\mathbf V_{\bmphi,t}
\left(
\mathbf S_t \bmphi^\star
+
\boldsymbol\xi_t
\right).
\]

Since $\mathbf V_{\bmphi,t} \mathbf S_t
=
\mathbf I - \mathbf V_{\bmphi,t}$,
we obtain
\[
\hat{\bmphi}_t - \bmphi^\star
=
\mathbf V_{\bmphi,t} \boldsymbol\xi_t
-
\mathbf V_{\bmphi,t} \bmphi^\star.
\]

We define $\mathbf V_{\bmphi,t}^{-1}$-norm for any vector $\mathbf x$ as $\|x\|_{\mathbf V_{\bmphi,t}^{-1}}
=
\sqrt{x^\top \mathbf V_{\bmphi,t}^{-1} x}$. Since $\mathbf V_{\bmphi,t} \succ 0$, this defines a valid norm. We apply the triangle inequality; since $\|\cdot\|_{\mathbf V_{\bmphi,t}^{-1}}$ is a norm, the triangle inequality applied on the expression above leads to
\[
\|
\hat{\bmphi}_t - \bmphi^\star
\|_{\mathbf V_{\bmphi,t}^{-1}} = \| \mathbf V_{\bmphi,t} \boldsymbol\xi_t - \mathbf V_{\bmphi,t} \bmphi^\star\|_{\mathbf{V}_{\bmphi,t}^{-1}}
\le
\|
\mathbf V_{\bmphi,t} \boldsymbol\xi_t
\|_{\mathbf V_{\bmphi,t}^{-1}}
+
\|
\mathbf V_{\bmphi,t} \bmphi^\star
\|_{\mathbf V_{\bmphi,t}^{-1}}.
\]

Since, for any vector $\mathbf x$, $\|\mathbf V_{\bmphi,t} \mathbf x
\|_{\mathbf V_{\bmphi,t}^{-1}}
=
\|\mathbf x\|_{\mathbf V_{\bmphi,t}}$, we simplify the above as
\begin{align}
\|
\hat{\bmphi}_t - \bmphi^\star
\|_{\mathbf V_{\bmphi,t}^{-1}}
\le
\|
\mathbf V_{\bmphi,t} \boldsymbol\xi_t
\|_{\mathbf V_{\bmphi,t}^{-1}}
+
\|\bmphi^\star\|_{\mathbf V_{\bmphi,t}}.
\label{eq:intbound_wmwstar}
\end{align}

Using $\mathbf V_{\bmphi,t}
=
(\mathbf I + \mathbf S_t)^{-1}
\preceq
\mathbf S_t^{-1}$, 
we have
\begin{align}
\|
\mathbf V_{\bmphi,t} \boldsymbol\xi_t
\|_{\mathbf V_{\bmphi,t}^{-1}}^2
=
\boldsymbol\xi_t^\top
\mathbf V_{\bmphi,t}
\boldsymbol\xi_t\le
\boldsymbol\xi_t^\top
\mathbf S_t^{-1}
\boldsymbol\xi_t.
\end{align}

Let $\mathbf u_t
=
\mathbf S_t^{-1/2}
\boldsymbol\xi_t$. Then $\mathbf u_t
\sim
\mathcal N(0,\sigma^2 \mathbf I_d)$, 
and $\boldsymbol\xi_t^\top
\mathbf S_t^{-1}
\boldsymbol\xi_t
=
\mathbf u_t^\top \mathbf u_t$. Evidently, $\frac{1}{\sigma^2}
\mathbf u_t^\top \mathbf u_t
\sim
\chi^2_d$, where $\chi^2_d$ refers to Chi-square distribution with $d$ degrees of freedom. Applying the Laurent-Massart inequality (Lemma~1 in \cite{laurentmassart}) we have with probability at least $1-\delta$,
\[
\mathbf u_t^\top \mathbf u_t
\le
\sigma^2
\left(
d
+
2\sqrt{d \log \tfrac{1}{\delta}}
+
2 \log \tfrac{1}{\delta}
\right).
\]

Substituting the above in \eqref{eq:intbound_wmwstar}, we obtain the final result. 
\hfill $\square$
\end{proof}

\noindent\textbf{Corollary 1 and Proof}
\emph{
Under the assumptions of Theorem~2, suppose that the minimum eigenvalue of
$S_t$ satisfies $\lambda_{\min}(S_t)\ge\kappa t$ for some $\kappa>0$. Then for any
$\nu>0$ and $\delta\in(0,1)$, with probability at least $1-\delta$,
$\|\hat\phi_t-\phi^\star\|_2\le\nu$ whenever
\begin{equation}
t \;\ge\; \frac{1}{\kappa}\left(\frac{(\beta_\delta+\|\phi^\star\|_2)^2}{\nu^2}-1\right),
\end{equation}
where $\beta_\delta=\sigma\sqrt{\,d+2\sqrt{d\log\tfrac{1}{\delta}}+2\log\tfrac{1}{\delta}\,}$.}

\begin{proof}
From Theorem~2, with probability at least $1-\delta$,
\[
\|\hat\phi_t-\phi^\star\|_{V_{\phi,t}^{-1}}
  \;\le\; \beta_\delta + \|\phi^\star\|_{V_{\phi,t}}.
\]
Since $V_{\phi,t}=(I+S_t)^{-1}\preceq I$, we have $\|\phi^\star\|_{V_{\phi,t}}\le\|\phi^\star\|_2$, and hence
\begin{equation}
\|\hat\phi_t-\phi^\star\|_{V_{\phi,t}^{-1}}
  \;\le\; \beta_\delta + \|\phi^\star\|_2.
\end{equation}
For the symmetric matrix $V_{\phi,t}\preceq I$, we have
$V_{\phi,t}\preceq\lambda_{\max}(V_{\phi,t})\,I$, and hence
$\lambda_{\max}(V_{\phi,t})\,V_{\phi,t}^{-1}\succeq I$. Then, for any vector $x$,
\[
\|x\|_2^2 = x^\top x
  \;\le\; \lambda_{\max}(V_{\phi,t})\,x^\top V_{\phi,t}^{-1}x
  \;=\; \lambda_{\max}(V_{\phi,t})\,\|x\|_{V_{\phi,t}^{-1}}^2.
\]
Therefore
$\|\hat\phi_t-\phi^\star\|_2\le\sqrt{\lambda_{\max}(V_{\phi,t})}\,\|\hat\phi_t-\phi^\star\|_{V_{\phi,t}^{-1}}$.
Since $V_{\phi,t}=(I+S_t)^{-1}$, we have $\lambda_{\max}(V_{\phi,t})=\tfrac{1}{\lambda_{\min}(I+S_t)}$, and so
\begin{equation}
\|\hat\phi_t-\phi^\star\|_2
  \;\le\; \frac{\beta_\delta+\|\phi^\star\|_2}{\sqrt{\lambda_{\min}(I+S_t)}}.
\end{equation}
Using $\lambda_{\min}(S_t)\ge\kappa t$, we have $\lambda_{\min}(I+S_t)\ge 1+\kappa t$, hence
\begin{equation}
\|\hat\phi_t-\phi^\star\|_2
  \;\le\; \frac{\beta_\delta+\|\phi^\star\|_2}{\sqrt{1+\kappa t}}.
\end{equation}
For $\|\hat\phi_t-\phi^\star\|_2\le\nu$ to hold, a sufficient condition is
\[
\frac{\beta_\delta+\|\phi^\star\|_2}{\sqrt{1+\kappa t}}\;\le\;\nu.
\]
Solving for $t$ yields
\[
t \;\ge\; \frac{1}{\kappa}\left(\frac{(\beta_\delta+\|\phi^\star\|_2)^2}{\nu^2}-1\right).
\]
\end{proof}
\subsection{Proof of Theorem 3}
\begin{lemma}[Prior distribution of $\bmphi$]
\label{lemma1_foreq3}
Let\, $\bmphi \mid \sigma^2 \sim \mathcal{N}(\mathbf{0},\,\sigma^2 \mathbf{I}_d)$
and\, $\sigma^2 \sim \operatorname{Scaled\text{-}Inv\text{-}}\chi^2(n_0,\,\sigma_0^2)$.
Then the marginal prior obtained by integrating out $\sigma^2$ is a
multivariate Student-$t$:
\begin{align*}
    p(\bmphi)
    \;=\;
    \int p(\bmphi \mid \sigma^2)\,p(\sigma^2)\,d\sigma^2
    \;=\;
    \mathrm{T}_{n_0}\!\bigl(\mathbf{0},\;\sigma_0^2\,\mathbf{I}_d\bigr).
\end{align*}
\end{lemma}

\begin{proof}
The conditional distribution is
\begin{align}
    p(\bmphi \mid \sigma^2)
    &= (2\pi)^{-d/2}\,
       (\sigma^2)^{-d/2}\,
       \exp\!\Bigl(-\tfrac{1}{2\sigma^2}\,\bmphi^\top\bmphi\Bigr),
    \label{eq:prior_cond}
\end{align}
and the prior on the noise variance is
\begin{align}
    p(\sigma^2)
    &= \frac{(n_0\sigma_0^2/2)^{n_0/2}}{\Gamma(n_0/2)}\;
       (\sigma^2)^{-(1+n_0/2)}\;
       \exp\!\Bigl(-\tfrac{n_0\sigma_0^2}{2\sigma^2}\Bigr).
    \label{eq:prior_sigma}
\end{align}
Multiplying \eqref{eq:prior_cond} and \eqref{eq:prior_sigma} and
collecting terms in $\sigma^2$:
\begin{align*}
    p(\bmphi \mid \sigma^2)\,p(\sigma^2)
    &= C\;
       (\sigma^2)^{-(1+(n_0+d)/2)}\;
       \exp\!\Bigl(-\tfrac{A}{2\sigma^2}\Bigr),
\end{align*}
where $A = \bmphi^\top\bmphi + n_0\sigma_0^2$ and $C$ collects all
factors independent of $\sigma^2$.

To integrate out $\sigma^2$, substitute $u = A/(2\sigma^2)$,
so that $d\sigma^2 = -A\,u^{-2}/2\;du$.  Then
\begin{align*}
    p(\bmphi)
    &= \int_0^\infty
       C\;\Bigl(\tfrac{A}{2u}\Bigr)^{-(1+(n_0+d)/2)}
       e^{-u}\;\tfrac{A}{2}\,u^{-2}\;du
    \;\propto\;
    A^{-(n_0+d)/2}.
\end{align*}
The integral over $u$ is a Gamma integral and evaluates to a constant.
Substituting back $A = \bmphi^\top\bmphi + n_0\sigma_0^2$:
\begin{align}
    p(\bmphi)
    &\;\propto\;
    \bigl(\bmphi^\top\bmphi + n_0\sigma_0^2\bigr)^{-(n_0+d)/2}
    \;=\;
    \Bigl(1 + \tfrac{\bmphi^\top\bmphi}{n_0\sigma_0^2}\Bigr)^{-(n_0+d)/2}
    \cdot (n_0\sigma_0^2)^{-(n_0+d)/2}.
    \label{eq:prior_kernel}
\end{align}
Recognising~\eqref{eq:prior_kernel} as the kernel of a $d$-dimensional
Student-$t$ with $n_0$ degrees of freedom, location $\mathbf{0}$,
and scale matrix $\sigma_0^2\,\mathbf{I}_d$, we conclude
$p(\bmphi) = \mathrm{T}_{n_0}(\mathbf{0},\;\sigma_0^2\,\mathbf{I}_d)$.
\end{proof}

\begin{lemma}[Posterior distribution of $\bmphi$]
\label{lemma2}
Given the Bayesian linear surrogate in~\eqref{eq:bayesian_surrogate}
with perturbations $\matZ$ and black-box responses
$\vecY = [y_1,\dots,y_N]^\top$, the posterior obtained by
marginalizing $\sigma^2$ is a multivariate Student-$t$:
\begin{align}
    p(\bmphi \mid \mathcal{Z},\vecY)
    &= \int p(\bmphi \mid \sigma^2,\mathcal{Z},\vecY)\;
       p(\sigma^2 \mid \mathcal{Z},\vecY)\;d\sigma^2
    \;=\;
    \mathrm{T}_{n_0+N}\!\bigl(\hat{\bmphi},\;
    \tfrac{\nu_1}{n_0+N}\,\mathbf{V}_{\bmphi}\bigr),
\end{align}
where $\nu_1 = n_0\sigma_0^2 + s^2$, and $\hat{\bmphi}$,
$\mathbf{V}_{\bmphi}$, $s^2$ are defined in~\eqref{eq:covarianceBLR}.
The conditional distributions are
\begin{align}
    \bmphi \mid \sigma^2,\mathcal{Z},\vecY
    &\sim \mathcal{N}\!\bigl(\hat{\bmphi},\;
          \sigma^2\,\mathbf{V}_{\bmphi}\bigr),
    \label{eq:post_cond_phi}\\[4pt]
    \sigma^2 \mid \mathcal{Z},\vecY
    &\sim \operatorname{Scaled\text{-}Inv\text{-}}\chi^2\!
          \Bigl(n_0+N,\;\tfrac{\nu_1}{n_0+N}\Bigr).
    \label{eq:post_sigma}
\end{align}
\end{lemma}

\begin{proof}
The conditional posterior~\eqref{eq:post_cond_phi} gives
\begin{align}
    p(\bmphi \mid \sigma^2,\mathcal{Z},\vecY)
    &= (2\pi)^{-d/2}\,
       \det(\sigma^2\mathbf{V}_{\bmphi})^{-1/2}\;
       \exp\!\Bigl(
         -\tfrac{1}{2\sigma^2}\,
         (\bmphi-\hat{\bmphi})^\top
         \mathbf{V}_{\bmphi}^{-1}
         (\bmphi-\hat{\bmphi})
       \Bigr),
    \label{eq:post_gauss}
\end{align}
and the marginal posterior on $\sigma^2$~\eqref{eq:post_sigma} gives
\begin{align}
    p(\sigma^2 \mid \mathcal{Z},\vecY)
    &= \frac{(\nu_1/2)^{(n_0+N)/2}}
            {\Gamma\!\bigl((n_0+N)/2\bigr)}\;
       (\sigma^2)^{-(1+(n_0+N)/2)}\;
       \exp\!\Bigl(-\tfrac{\nu_1}{2\sigma^2}\Bigr).
    \label{eq:post_sigma_explicit}
\end{align}
Multiplying~\eqref{eq:post_gauss} and~\eqref{eq:post_sigma_explicit}
and collecting powers of $\sigma^2$:
\begin{align}
    p(\bmphi \mid \sigma^2,\mathcal{Z},\vecY)\;
    p(\sigma^2 \mid \mathcal{Z},\vecY)
    &= C_1\;
       (\sigma^2)^{-(1+(n_0+N+d)/2)}\;
       \exp\!\Bigl(-\tfrac{A_1}{2\sigma^2}\Bigr),
    \label{eq:post_joint}
\end{align}
where
$A_1 = (\bmphi-\hat{\bmphi})^\top\mathbf{V}_{\bmphi}^{-1}
(\bmphi-\hat{\bmphi}) + \nu_1$
and $C_1$ absorbs all factors independent of $\sigma^2$.

Substituting $u = A_1/(2\sigma^2)$ and integrating out $\sigma^2$
(identical change-of-variables as in Lemma~\ref{lemma1_foreq3}):
\begin{align}
    p(\bmphi \mid \mathcal{Z},\vecY)
    &= \int_0^\infty
       C_1\;\Bigl(\tfrac{A_1}{2u}\Bigr)^{-(1+(n_0+N+d)/2)}
       e^{-u}\;\tfrac{A_1}{2}\,u^{-2}\;du
    \;\propto\;
    A_1^{-(n_0+N+d)/2}.
    \label{eq:post_marginal_kernel}
\end{align}
Substituting back:
\begin{align}
    p(\bmphi \mid \mathcal{Z},\vecY)
    &\;\propto\;
    \Bigl(
      1 + \tfrac{1}{\nu_1}\,
      (\bmphi-\hat{\bmphi})^\top
      \mathbf{V}_{\bmphi}^{-1}
      (\bmphi-\hat{\bmphi})
    \Bigr)^{-(n_0+N+d)/2}.
    \label{eq:post_t_kernel}
\end{align}
Expression~\eqref{eq:post_t_kernel} is the kernel of a $d$-dimensional
Student-$t$ with $n_0+N$ degrees of freedom, location $\hat{\bmphi}$,
and scale matrix $\tfrac{\nu_1}{n_0+N}\,\mathbf{V}_{\bmphi}$.
Hence
$p(\bmphi \mid \mathcal{Z},\vecY) =
\mathrm{T}_{n_0+N}\!\bigl(\hat{\bmphi},\;
\tfrac{\nu_1}{n_0+N}\,\mathbf{V}_{\bmphi}\bigr)$.
\end{proof}

\begin{theorem}
\label{th:A_eig}
Given a Bayesian linear regression surrogate and a set of perturbations
\( \matZ = [z_1,\dots,z_N] \) with corresponding predictions \( \mathbf y \),
the expected information gain over the surrogate importance parameters
\( \bmphi \) is given by
\begin{align}
    \mathcal A(\matZ; f_e)
    =
    \tfrac{1}{2} \log \lvert \mathbf V_{\bmphi} \rvert
    + \log \mathcal{C}_1
    - \log \mathcal{C}_2,
\end{align}
where \( \mathbf V_{\bmphi} \) and the weight matrix \( \bmphi \) are
defined as in Section~\ref{sec: problem_setting},
\( r_1 = \tfrac{1}{2}(n_0 + N) \),
and \( \nu_1 = n_0\sigma_0^2 + s^2 \).
The constants \( \mathcal{C}_1 \) and \( \mathcal{C}_2 \) are given by
\begin{align*}
    \mathcal{C}_1
    &=
    \frac{(\pi \nu_{1})^{r_{1}/2}}{\Gamma(r_{1}/2)}
    \, \mathrm{B}\!\left(\tfrac{r_{1}}{2}, \tfrac{\nu_{1}}{2}\right)
    + \frac{\nu_{1}+r_{1}}{2}
    \Biggl[
        \psi\!\left(\tfrac{\nu_{1}+r_{1}}{2}\right)
        - \psi\!\left(\tfrac{\nu_{1}}{2}\right)
    \Biggr], \\[6pt]
    \mathcal{C}_2
    &=
    \frac{(n_0\sigma_0^2\pi)^{n_0/2}}
    {\Gamma(n_0/2)}
    \, \mathrm{B}\!\left(\tfrac{n_0}{2}, \tfrac{n_0\sigma_0^2}{2}\right)
    + \frac{n_0}{2}
    \Biggl[
        \psi\!\left(\tfrac{n_0}{2}\right)
        - \psi\!\left(\tfrac{n_0\sigma_0^2}{2}\right)
    \Biggr].
\end{align*}
\end{theorem}
\begin{proof}
The expected information gain is defined as
\begin{align}
    \mathcal{A}(\matZ;\,f_e)
    &= \mathbb{E}_{p(\vecY \mid \matZ)}\!
       \bigl[
         \mathrm{H}[p(\bmphi)]
         - \mathrm{H}[p(\bmphi \mid \mathcal{Z},\vecY)]
       \bigr].
    \label{eq:eig_def}
\end{align}

Let $\mathbf{x} \sim \mathrm{T}_\nu(\boldsymbol{\mu},\,\boldsymbol{\Sigma})$
be a $d$-dimensional Student-$t$ with $\nu$ degrees of freedom,
location $\boldsymbol{\mu}$, and scale matrix $\boldsymbol{\Sigma}$.
Its differential entropy is
\begin{align}
    \mathrm{H}[\mathbf{x}]
    &= \tfrac{1}{2}\log\lvert\boldsymbol{\Sigma}\rvert
       + \log\!\biggl[
         \frac{(\nu\pi)^{d/2}}{\Gamma(d/2)}\,
         \mathrm{B}\!\Bigl(\tfrac{d}{2},\,\tfrac{\nu}{2}\Bigr)
         + \tfrac{\nu+d}{2}\,
         \Bigl(\psi\!\bigl(\tfrac{\nu+d}{2}\bigr)
               - \psi\!\bigl(\tfrac{\nu}{2}\bigr)\Bigr)
       \biggr].
    \label{eq:entropy_t}
\end{align}

From Lemma~\ref{lemma1_foreq3},
$p(\bmphi) = \mathrm{T}_{n_0}(\mathbf{0},\;\sigma_0^2\,\mathbf{I}_d)$.
Applying~\eqref{eq:entropy_t} with $\nu = n_0$ and
$\bm{\Sigma} = \sigma_0^2\,\mathbf{I}_d$:
\begin{align}
    \mathrm{H}[p(\bmphi)]
    &= \tfrac{d}{2}\log\sigma_0^2
       + \log\!\biggl[
         \frac{(n_0\pi)^{d/2}}{\Gamma(d/2)}\,
         \mathrm{B}\!\Bigl(\tfrac{d}{2},\,\tfrac{n_0}{2}\Bigr)
         + \tfrac{n_0+d}{2}\,
         \Bigl(\psi\!\bigl(\tfrac{n_0+d}{2}\bigr)
               - \psi\!\bigl(\tfrac{n_0}{2}\bigr)\Bigr)
       \biggr].
    \label{eq:H_prior}
\end{align}
Since $n_0$ and $\sigma_0^2$ are fixed hyperparameters, we write
$\mathrm{H}[p(\bmphi)] = \tfrac{d}{2}\log\sigma_0^2 + \log\mathcal{C}_2$.

From Lemma~\ref{lemma2},
$p(\bmphi \mid \mathcal{Z},\vecY) =
\mathrm{T}_{n_0+N}\!\bigl(\hat{\bmphi},\;
\tfrac{\nu_1}{n_0+N}\,\mathbf{V}_{\bmphi}\bigr)$
with $\nu_1 = n_0\sigma_0^2 + s^2$.
Applying~\eqref{eq:entropy_t} with $\nu = n_0+N$ and
$\boldsymbol{\Sigma} = \tfrac{\nu_1}{n_0+N}\,\mathbf{V}_{\bmphi}$:
\begin{align}
    \mathrm{H}[p(\bmphi \mid \mathcal{Z},\vecY)]
    &= \tfrac{1}{2}\log\!\bigl\lvert
         \tfrac{\nu_1}{n_0+N}\,\mathbf{V}_{\bmphi}
       \bigr\rvert
       + \log\!\biggl[
         \frac{((n_0+N)\pi)^{d/2}}{\Gamma(d/2)}\,
         \mathrm{B}\!\Bigl(\tfrac{d}{2},\,\tfrac{n_0+N}{2}\Bigr)
         \notag\\
    &\qquad\qquad
         + \tfrac{n_0+N+d}{2}\,
         \Bigl(\psi\!\bigl(\tfrac{n_0+N+d}{2}\bigr)
               - \psi\!\bigl(\tfrac{n_0+N}{2}\bigr)\Bigr)
       \biggr] \notag\\[4pt]
    &= \tfrac{1}{2}\log\lvert\mathbf{V}_{\bmphi}\rvert
       + \tfrac{d}{2}\log\!\bigl(\tfrac{\nu_1}{n_0+N}\bigr)
       + \log\mathcal{C}_1',
    \label{eq:H_post}
\end{align}
where $\mathcal{C}_1'$ denotes the bracketed expression evaluated at
$\nu = n_0+N$.  We absorb the scaling into
$\log\mathcal{C}_1$.

The $\log\lvert\mathbf{V}_{\bmphi}\rvert$ term depends on $\matZ$
but not on $\vecY$, and hence, it passes
through the expectation in~\eqref{eq:eig_def}.
Substituting Steps~2 and~3 into~\eqref{eq:eig_def}:
\begin{align}
    \mathcal{A}(\matZ;\,f_e)
    &= \mathbb{E}_{p(\vecY \mid \matZ)}\!
       \bigl[
         \mathrm{H}[p(\bmphi)]
         - \mathrm{H}[p(\bmphi \mid \mathcal{Z},\vecY)]
       \bigr]
    \notag\\[2pt]
    &= \bigl(\tfrac{d}{2}\log\sigma_0^2 + \log\mathcal{C}_2\bigr)
       -
       \bigl(\tfrac{1}{2}\log\lvert\mathbf{V}_{\bmphi}\rvert
             + \tfrac{d}{2}\log\!\bigl(\tfrac{\nu_1}{n_0+N}\bigr)
             + \log\mathcal{C}_1\bigr)
    \notag\\[2pt]
    &= -\,\tfrac{1}{2}\log\lvert\mathbf{V}_{\bmphi}\rvert
       + \log\mathcal{C}_2
       - \log\mathcal{C}_1
       + \tfrac{d}{2}\log\sigma_0^2
       - \tfrac{d}{2}\log\!\bigl(\tfrac{\nu_1}{n_0+N}\bigr).
    \label{eq:eig_final}
\end{align}
The last two scalar terms depend only on fixed hyperparameters and the
data-dependent quantity $\nu_1$.  In the main text these are absorbed
into the constants $\mathcal{C}_1$ and $\mathcal{C}_2$, yielding
the stated result.
\hfill $\square$
\end{proof}

\newpage
\section{Additional Experimental Results}

\subsection{Sampling Quality}
\textbf{A-efficiency} is defined as $\operatorname{Tr}(V_0) / \operatorname{Tr}(V_t)$, captures the reduction in total marginal variance of the surrogate coefficients. While D-efficiency summarizes the joint posterior volume, A-optimality seeks to minimize the trace of the inverse of the information matrix, resulting in minimizing the average variance of the estimates of the regression coefficients. This makes A-efficiency particularly sensitive to individual poorly estimated features, flagging acquisition strategies that leave some coefficients under-constrained even when the overall posterior volume is small. Higher values indicate greater average precision across all coefficient estimates. 
\texttt{EAGLE} achieves consistently higher A-efficiency than both baselines across all datasets, with the gap widening over the perturbation budget. The ECDF plots in  Fig.~\ref{fig:a-eff-ecdf} confirm that at $n=500$, nearly all test instances under \texttt{EAGLE} attain higher A-efficiency than the best instances under BayesLIME or BayLIME, indicating that the improvement is uniform across instances rather than driven by a few outliers.

\begin{figure}[t]
    \centering
    
    \includegraphics[width=\textwidth]{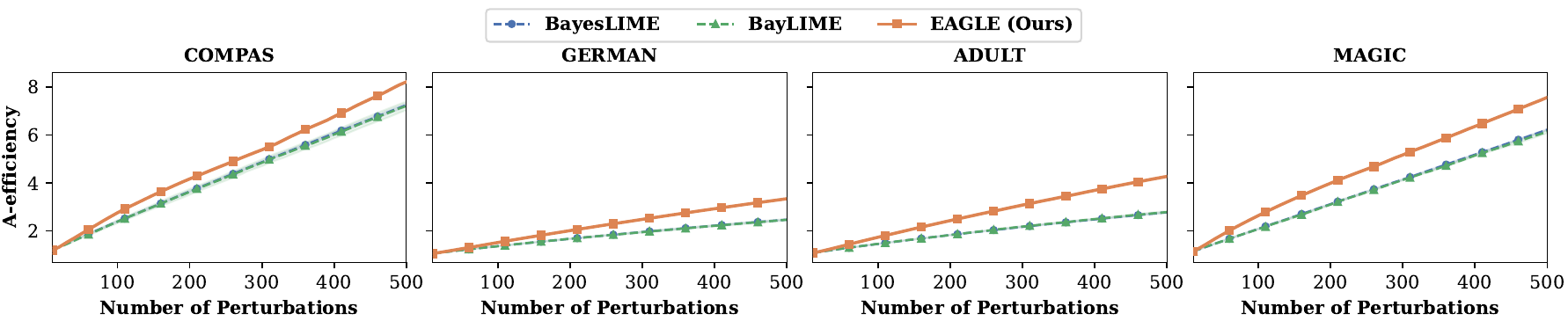}

    \caption{Convergence of A-efficiency as a function of perturbation budget across tabular datasets. 
    \texttt{EAGLE} consistently achieves higher A-efficiency with fewer perturbations compared to BayesLIME and BayLIME, indicating more informative perturbation selection.}
    
    \label{fig:a-eff-convergence}
\end{figure}

\begin{figure}[t]
    \centering  
    \includegraphics[width=\textwidth]{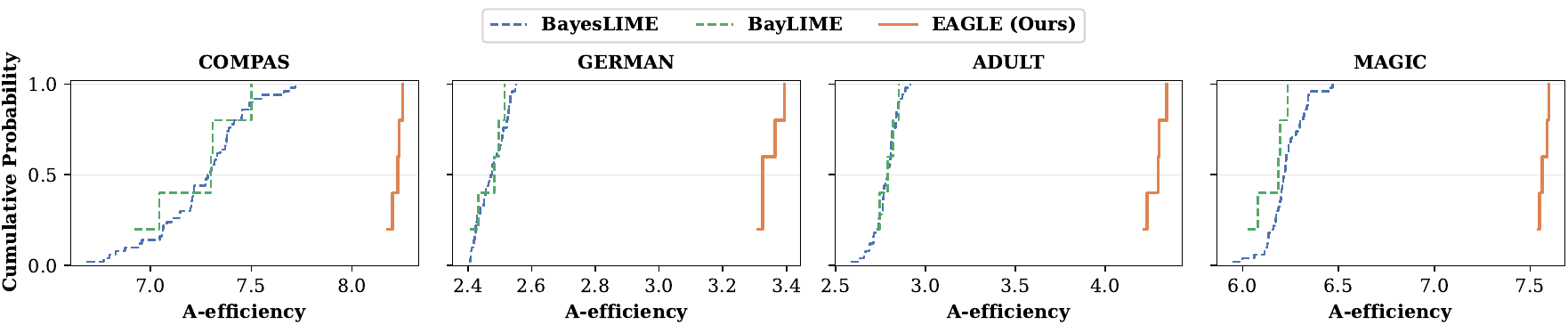}
    \caption{Empirical cumulative distribution functions (ECDF) of A-efficiency at $n{=}500$ across test instances. 
    Curves further to the right indicate higher information gain. \texttt{EAGLE} consistently dominates the baselines across all datasets.}    
    \label{fig:a-eff-ecdf}
\end{figure}





\begin{figure}[t]
\centering
\includegraphics[width=\textwidth]{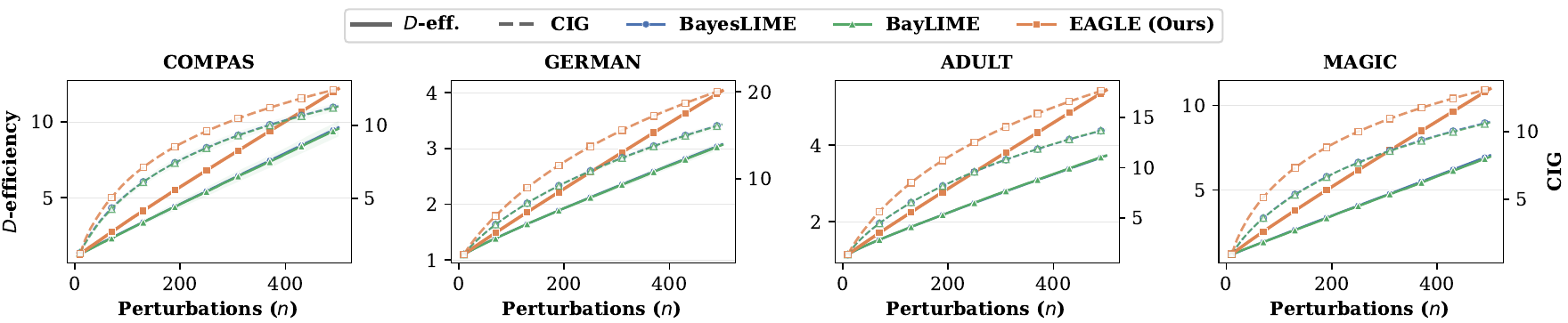}
\caption{Convergence of D-efficiency (solid lines) and cumulative information gain (CIG, dashed lines) as a function of perturbation budget $n$ across tabular datasets. Results are averaged over test instances. \texttt{EAGLE} consistently achieves higher efficiency with fewer perturbations compared to the baselines.}
\label{fig:deff-convergence}
\end{figure}




\begin{figure}[t]
\centering
\includegraphics[width=\textwidth]{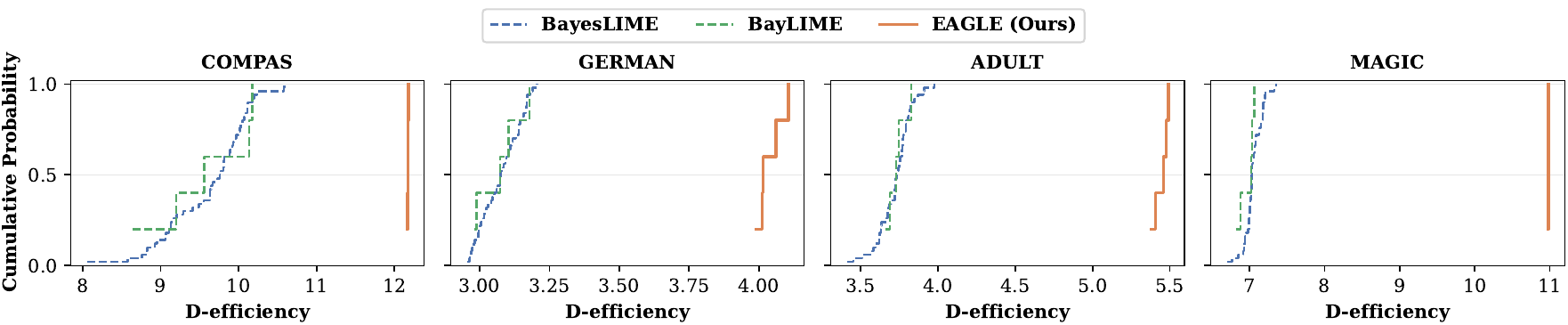}
\caption{Empirical cumulative distribution functions (ECDF) of D-efficiency at $n=500$ across test instances for all tabular datasets. Curves further to the right indicate higher efficiency. \texttt{EAGLE} consistently dominates the baselines.}
\label{fig:deff-ecdf}
\end{figure}







\subsection{Uncertainty in Explanations}

While Jaccard similarity measures whether explanations are consistent across runs, it treats all feature attributions equally regardless of how confident the model is in each estimate. For an explainer to be trustworthy, it must also convey which attributions are reliable and which remain uncertain. Bayesian surrogate methods uniquely provide this through posterior credible intervals. \textbf{Credible interval width ($90\%$)} reports the mean width of the $90\%$ credible intervals across all surrogate coefficients as shown in figures~\ref{fig:unc} and \ref{fig:text_unc}. 
Narrow intervals indicate confident separation of important features from unimportant ones, while wide and overlapping intervals indicate that attributions are poorly estimated, making it difficult to distinguish genuinely important features from noise. Figure~\ref{fig:unc} shows per-feature attributions with $90\%$ credible intervals for randomly chosen instances from the German Credit and Adult Income datasets, where \texttt{EAGLE} consistently achieves the tightest intervals across all features, indicating that its acquisition strategy concentrates posterior more effectively than the baselines. Figure~\ref{fig:text_unc} illustrates \texttt{EAGLE} on a spam classification task, where each word's color intensity encodes the magnitude of its attribution (green for positive, red for negative contribution toward SPAM) and border thickness encodes the $90\%$ credible interval width. Together, these results confirm that \texttt{EAGLE}'s information gain driven sampling not only improves point estimate stability (Table~\ref{tab:jaccard5}) but also yields tighter uncertainty quantification across modalities.
\begin{figure}[t]
    \centering   
    \includegraphics[width=\linewidth]{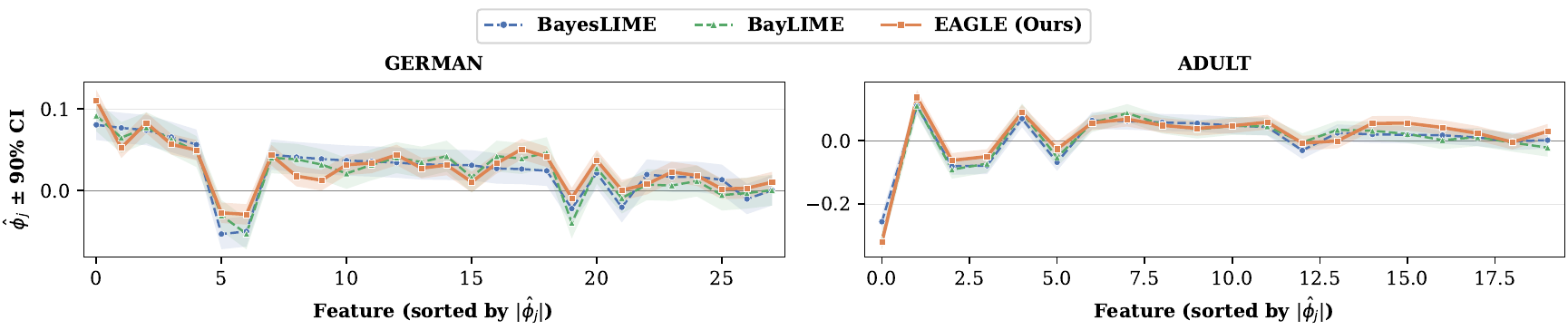}
    
     \caption{Per-feature attribution estimates with $90\%$ credible intervals (shaded regions), sorted by magnitude, for a randomly selected instance from the German Credit and Adult Income datasets. \texttt{EAGLE} consistently achieves the tightest credible intervals across features, indicating lower uncertainty in its explanations compared to BayesLIME and BayLIME}
     \label{fig:unc}
\end{figure}

\begin{figure}[t]
    \centering
    \includegraphics[width=\linewidth,
    trim=0cm 0.6cm 0cm 0.6cm,clip]{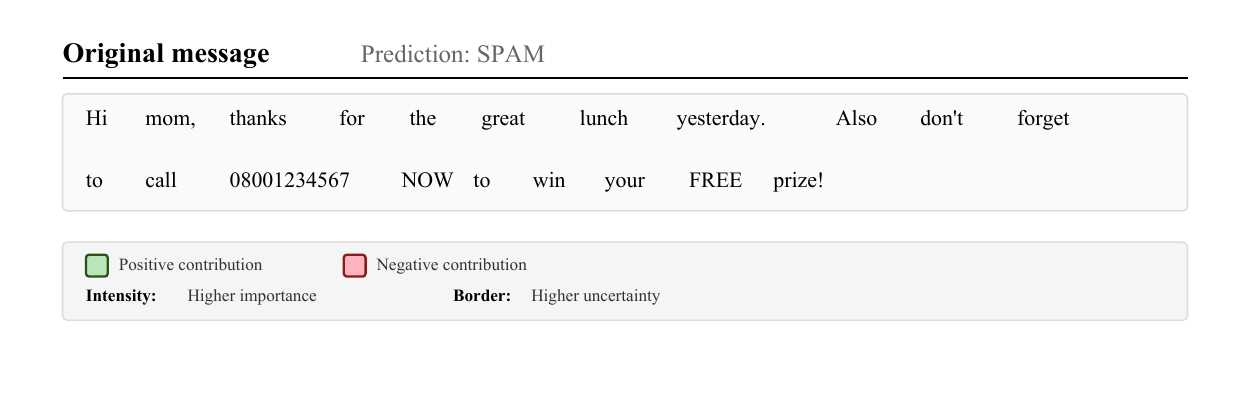}    
    \includegraphics[width=\linewidth,
    trim=0cm 0.6cm 0cm 0.6cm,clip]{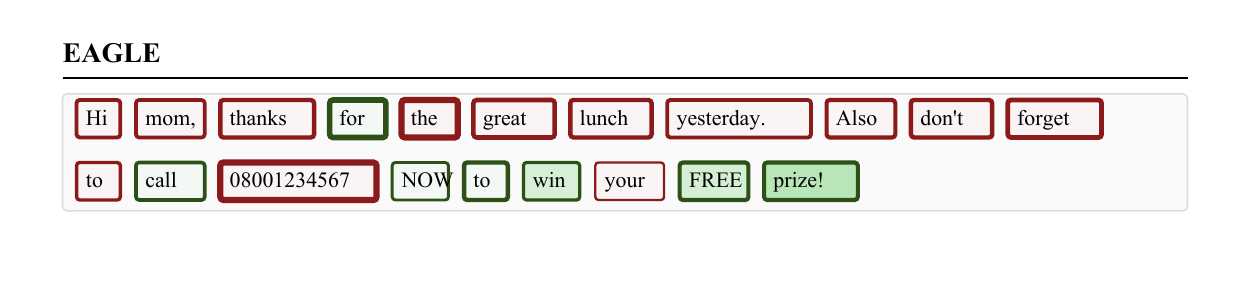}   
    \caption{Visualization of explanations obtained using \texttt{EAGLE} for text spam detection dataset. Green highlights positive contributions toward the SPAM prediction, red indicates negative contributions, and border thickness represents attribution uncertainty.}
    \label{fig:text_unc}
\end{figure}

\end{document}